\documentclass{article}

\PassOptionsToPackage{dvipsnames}{xcolor}

\usepackage{microtype}
\usepackage{graphicx}
\usepackage{subcaption}
\usepackage{booktabs}
\usepackage{hyperref}
\usepackage[preprint]{icml2026}

\usepackage{amsmath,amssymb,mathtools,amsthm}
\usepackage[capitalize,noabbrev]{cleveref}

\usepackage{algorithm,algorithmic}
\usepackage{float,wrapfig,multirow,colortbl}

\theoremstyle{plain}
\newtheorem{theorem}{Theorem}[section]

\newtheorem{lemma}[theorem]{Lemma}

\theoremstyle{definition}
\newtheorem{definition}[theorem]{Definition}
\newtheorem{assumption}[theorem]{Assumption}
\theoremstyle{remark}

\usepackage{bm}

\def\eqref#1{equation~\ref{#1}}

\def\1{\bm{1}}

\DeclareMathAlphabet{\mathsfit}{\encodingdefault}{\sfdefault}{m}{sl}
\SetMathAlphabet{\mathsfit}{bold}{\encodingdefault}{\sfdefault}{bx}{n}

\icmltitlerunning{Zero-Forgetting CISS via Dual-Phase Cognitive Cascades}

\begin{document}

\twocolumn[
  \icmltitle{Zero-Forgetting Class-Incremental Segmentation via \\
  Dual-Phase Cognitive Cascades}

  \begin{icmlauthorlist}
    \icmlauthor{Yuquan Lu\textsuperscript{*}}{sysu}
    \icmlauthor{Yifu Guo\textsuperscript{*}}{sysu}
    \icmlauthor{Zishan Xu}{sjtu}
    \icmlauthor{Siyu Zhang}{swu}
    \icmlauthor{Yu Huo}{cuhk}
    \icmlauthor{Siyue Chen}{scut}
    \icmlauthor{Siyan Wu}{scnu}
    \icmlauthor{Chenghua Zhu}{scnu}
    \icmlauthor{Ruixuan Wang}{sysu}
  \end{icmlauthorlist}

  \icmlaffiliation{sysu}{Sun Yat-sen University}
  \icmlaffiliation{sjtu}{Shanghai Jiao Tong University}
  \icmlaffiliation{swu}{Southwest University}
  \icmlaffiliation{cuhk}{The Chinese University of Hong Kong}
  \icmlaffiliation{scut}{South China University of Technology}
  \icmlaffiliation{scnu}{South China Normal University}

  \icmlcorrespondingauthor{Ruixuan Wang}{wangruix5@mail.sysu.edu.cn}

  \vskip 0.3in
]

\printAffiliationsAndNotice{\icmlEqualContribution}

\begin{abstract}
  Continual semantic segmentation (CSS) is a cornerstone task in
  computer vision that enables a large number of downstream
  applications, but faces the catastrophic forgetting challenge.
  In conventional class-incremental semantic segmentation (CISS)
  frameworks using Softmax-based classification heads, catastrophic
  forgetting originates from Catastrophic forgetting and task
  affiliation probability.
  We formulate these problems and provide a theoretical analysis to
  more deeply understand the limitations in existing CISS methods,
  particularly Strict Parameter Isolation (SPI).
  To address these challenges, we follow a dual-phase intuition from
  human annotators, and introduce \textbf{Cog}nitive \textbf{Ca}scade
  \textbf{S}egmentation (CogCaS), a novel dual-phase cascade
  formulation for CSS tasks in the CISS setting.
  By decoupling the task into class-existence detection and
  class-specific segmentation, CogCaS enables more effective
  continual learning, preserving previously learned knowledge while
  incorporating new classes.
  Using two benchmark datasets PASCAL VOC 2012 and ADE20K, we have
  shown significant improvements in a variety of challenging
  scenarios, particularly those with long sequence of incremental
  tasks, when compared to exsiting state-of-the-art methods.
  Our code will be made publicly available upon paper acceptance.
\end{abstract}

\section{Introduction}

Deep learning has driven rapid progress across a wide range of
tasks, from natural language
understanding~\cite{du2025graphmaster,du2025graphoracle} to visual
perception, among which semantic segmentation has become a
cornerstone of computer vision with influential applications in
autonomous navigation and medical diagnostics. Despite these
advances, real-world deployment faces a fundamental limitation, i.e.,
the conventional paradigm requires all object categories to be
predefined before training. When new classes emerge, as they
inevitably do in dynamic environments, models must be retrained with
training data of both old and new classes, incurring increased
computational costs and potential privacy concerns. 

This limitation has driven substantial research in continual learning
(CL), where models incrementally acquire new knowledge while trying
to preserve existing capabilities. The core challenge is catastrophic
forgetting: the tendency of neural networks to abruptly forget
previously learned information when updated. Current approaches span
multiple paradigms, including regularization-based
methods~\cite{wang2022survey,xu2021revisiting} which constrain
parameter updates, replay-based strategies~\cite{aljundi2021end,
rusu2022end} which maintain historical exemplars, and optimization
techniques~\cite{li2022robust} which seek non-interfering parameter
spaces. Among these, SPI~\cite{serra2018overcoming} stands out for
providing theoretical guarantee of zero-forgetting  through parameter
compartmentalization, effectively addressing the stability-plasticity
dilemma in classification tasks.

However, extending continual learning to semantic segmentation
introduces unique complexities beyond classification. In particular
in class-incremental semantic segmentation (CISS), models must adapt
to new categories while maintaining pixel-precise understanding of
previous classes, all without access to complete historical data.
This setting introduces the challenge of background shift: pixels
belonging to future classes are temporarily labeled as background,
requiring the model to continuously revise its understanding of what
constitutes ``background'' as new classes
emerge~\cite{cermelliModelingBackgroundIncremental2020}. {This continuous re-evaluation of
  the ``background'' class due to background shift critically
  exacerbates the stability-plasticity dilemma, often compelling the
  model to make a detrimental trade-off between the forgetting of
  previously learned classes and the insufficient acquisition of new
ones, a challenge visually depicted in Figure~\ref{fig:error_linear}.}

While recently developed CISS approaches show promising performance
in knowledge preservation, we prove that this localized optimization
strategy, even assuming perfect preservation of knowledge of past
classes, structurally prevents convergence to the global optimum
achievable through joint training on all classes. Our theoretical
analysis also reveals that they face fundamental limitations.
Architecturally, the commonly used task headers based on Softmax need
to output complete probabilities for all current class at every
stage. However, the probability output for task segmentation head is
essentially local and is only optimized for the current task
category. 
This creates a stark trade-off: while freezing historical heads
eliminates dynamic interference, it simultaneously cements
distributional biases that prevent global optimization.



Motivated by these theoretical analyses, we propose a
\textbf{Cog}nitive \textbf{Ca}scade \textbf{S}egmentation (CogCaS)
architecture that fundamentally restructures the CISS paradigm. Our
approach 
introduces two key innovations, i.e., \textbf{Existence-Driven
Activation} by which segmentation heads are activated only for
detected classes and thus  background interference is eliminated, and
\textbf{Parameter Modularity} by which  independent detector and
segmenter per class enable isolated evolution without cross-task contamination.
Our implementation deliberately adopts elementary components: basic
backbones, single-layer detectors and simple segmenters, and standard
loss functions. This architectural transparency ensures that our
empirical improvements stem purely from the cognitive cascade design
rather than implementation sophistication, while maintaining the
framework's extensibility for future enhancements.

In summary, our main contributions include (1) the first systematic
characterization of the advantages and disadvantages of SPI strategy
in CISS, (2) a cognitively inspired architecture that resolves the
stability-plasticity dilemma through decoupling design, and (3)
state-of-the-art performance with significantly increased margins
particularly in challenging long-sequence continual learning scenarios.

\begin{figure}[!tbh]
  \centering
  \begin{minipage}{0.48\textwidth}
    \centering
    \includegraphics[width=\linewidth]{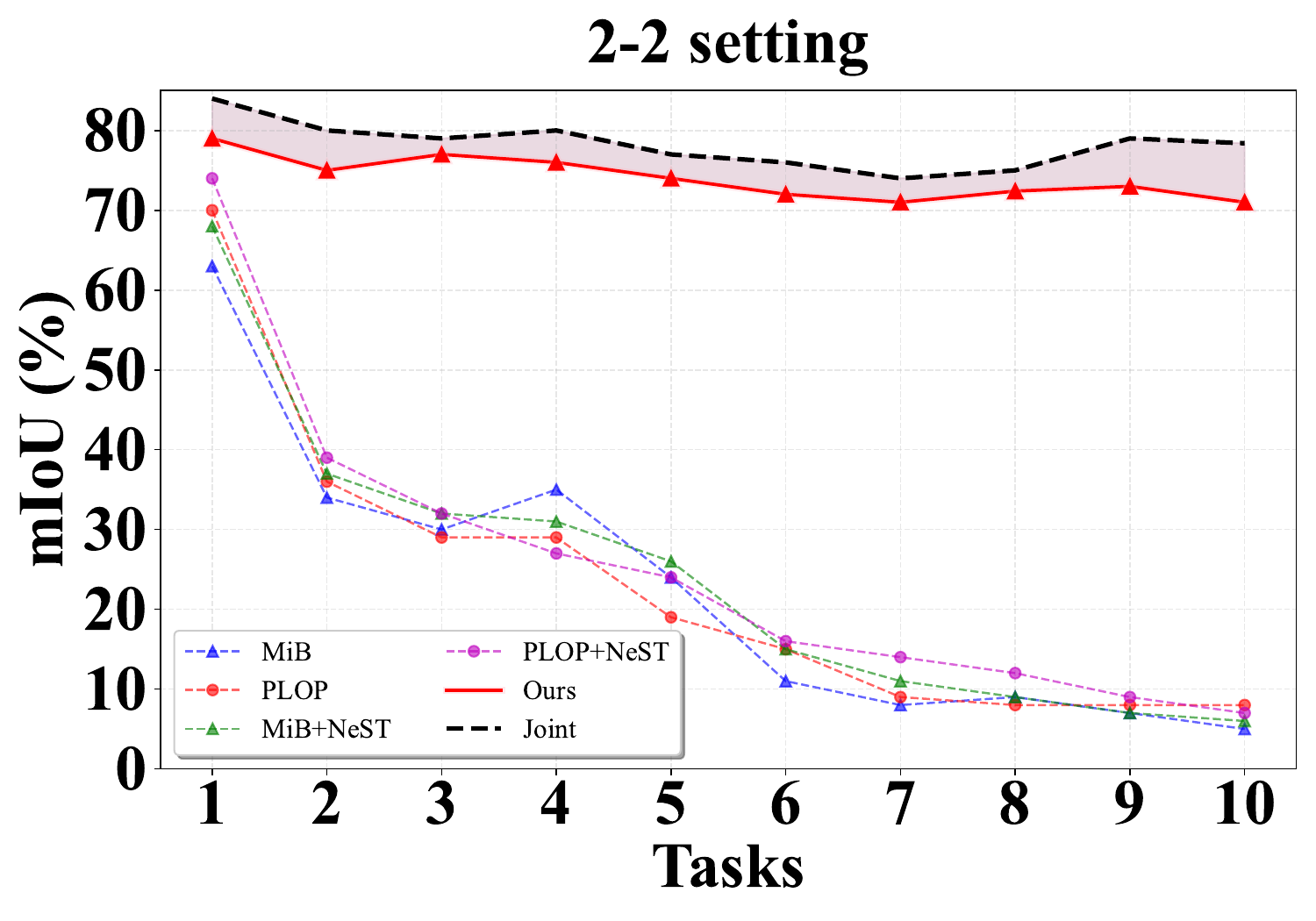} 
  \end{minipage}\hfill 
  \begin{minipage}{0.48\textwidth} 
    \centering 
    \includegraphics[width=\linewidth]{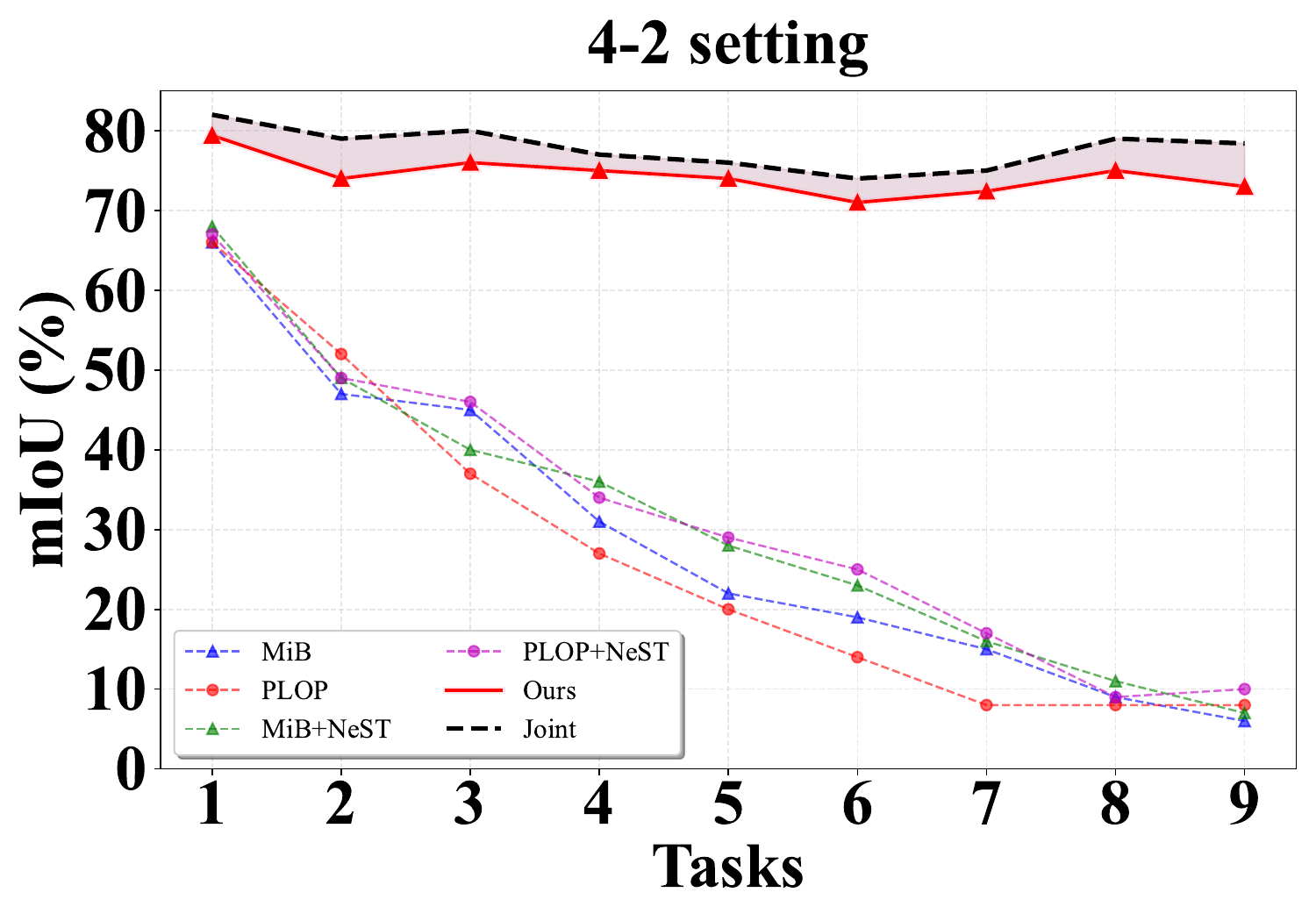}
  \end{minipage}

  \caption{ mIoU of classes learned at each task in PASCAL VOC
    continual semantic segmentation under 2-2 (left) and 4-2 (right)
    settings. While existing methods exhibit declining learning
    capability as tasks progress, our CogCaS consistently achieves
    strong performance around 70\% mloU,demonstrating stable and robust
  learning ability across all incremental stages.}
  \label{fig:error_linear}
\end{figure}

\section{Related Work}
Class-Incremental Semantic Segmentation (CISS) as a relatively new
research topic is often considered as an extension of image-level
continual learning, and therefore those techniques developed
originally for image-level continual learning have been adopted for
CISS, including \emph{regularization}~\cite{cermelliModelingBackgroundIncremental2020},
\emph{pseudo-labeling/self-supervision}~\cite{ECCV2024EarlyPP,yang2023distillation},
and \emph{experience
replay}~\cite{chaSSULSemanticSegmentation2021,samrep2021RECALL,Yang_Li_Lin_Wang_Lin_Liu_Wang_2023,zhangMiningUnseenClasses2022}.
Such adoption is still being actively explored, as shown by very
recent attempts~\cite{yu2025ipseg}. All these methods model
segmentation as \emph{per-pixel multi-class Softmax classification},
where old and new logits directly compete, causing background drift
that accumulate with the number of
tasks~\cite{farajtabar2020orthogonal} which can be seen in
Figure~\ref{fig:error_linear}.

To mitigate Softmax competition, several studies recast CISS as a
\emph{pixel-level multi-label} problem combined with strict parameter
freezing. SSUL~\cite{chaSSULSemanticSegmentation2021} appends binary channels for each
class and freezes them thereafter, but relies heavily on the initial
task's class capacity to initialize its ``unknown'' class prediction;
in dense incremental settings (e.g., 1-1), this causes performance
collapse as incomparable logits from disjoint phases are forced into
global competition. IPSeg~\cite{yu2025ipseg} extends this paradigm
with image posterior calibration but still requires exemplar replay
for stability. In addition, built on Mask2Former~\cite{chengPerPixelClassificationNot2021},
methods such as
CoMFormer~\cite{cermelliCoMFormerContinualLearning2023} and
CoMasTRe~\cite{CoMasTRe2024} treat CISS as a
\emph{mask-classification} problem. CoMasTRe decouples ``where'' and
``what'' via query-based segmenters and objectness transfer, yet
relies on complex multi-stage distillation to mitigate forgetting.
ECLIPSE~\cite{kim2024eclipse} employs visual prompt tuning within a
shared decoder and logit manipulation to address drift in continual
panoptic segmentation. \textbf{Our \emph{CogCaS} adopts
  \emph{existence-driven activation}, which decouples ``does class
  \(c\) exist?'' from ``segment class \(c\)'', and freezes each
  class-specific segmentation head after learning, achieving
  task-agnostic zero forgetting while maintaining high plasticity for
new classes.} Crucially, CogCaS is derived from Hessian-based
theoretical analysis (section 3.2): the explicit image-level router
activates independent binary heads, structurally eliminating
cross-task competition without distillation or replay.

\section{Methodology}

\subsection{Preliminaries}
\label{sec:prelim}

\noindent\textbf{Notations.} 
We consider the problem of Class-Incremental Semantic Segmentation
(CISS), where a model sequentially learns a sequence of $T$ tasks,
denoted as $\{\mathcal{T}_1, \dots, \mathcal{T}_T\}$. Each task
$\mathcal{T}_t$ is associated with a unique training set
$\mathcal{D}_t$ which contains certain number of training images and
corresponding pixel-wise annotations.  The task introduces a set of
task-specific foreground classes $\mathcal{C}_t$, which are mutually
exclusive between tasks, i.e., $\mathcal{C}_{\tau} \cap \mathcal{C}_t
= \emptyset$ for all $\tau \neq t$. For each training image in task
$\mathcal{T}_t$, each pixel is annotated as one of the task-specific
foreground classes in $\mathcal{C}_t$ or the special ``background''
class. Notably for the training set $\mathcal{D}_t$, all pixels not
belonging to the foreground class set $\mathcal{C}_t$ are annotated
as background. Thus, the annotated background regions in each
training image of task $\mathcal{T}_t$ encompass the true background
as well as image regions corresponding to foreground classes from
past ($\mathcal{T}_{1:t-1}$) or future ($\mathcal{T}_{t+1:T}$) tasks.
When updating the model with the training set $\mathcal{D}_t$ and
validation set $\mathcal{D}^{v}_t$ of task $\mathcal{T}_t$, the model
needs to be expanded to account for the set $\mathcal{C}_t$ of new
foreground classes. Let $\bm{\theta}$ denote the collection of all
the learnable model parameters throughout the continual learning
process (from task $\mathcal{T}_1$ to $\mathcal{T}_T$). When the model learns
task $\mathcal{T}_t$, the part of $\bm{\theta}$ which are uniquely
associated with future tasks ($\mathcal{T}_{t+1:T}$) are frozen with
certain default value (i.e., zero here). After updating the model
based on certain task-specific loss function $\mathcal{L}_t$, the
leranable model parameters are changed from $\bm{\theta}^*_{t-1}$ to
$\bm{\theta}^*_t$, where $\bm{\theta}^*_t$ represents the locally
optimal model parameters that minimizes $\mathcal{L}_t$. The change
in parameter values from $\mathcal{T}_{t-1}$ to $\mathcal{T}_t$ is
denoted by $\Delta_t:=\bm{\theta}_t^*-\bm{\theta}_{t-1}^*$, and the
Hessian matrix of the loss function $\mathcal{L}_t$ with respect to
the learnable parameters $\bm{\theta}$ is denoted by
$\mathbf{H}_t(\bm{\theta})=\frac{\partial^2\mathcal{L}_t(\bm{\theta})}{\partial\bm{\theta}^2}$.

\noindent\textbf{Convergence Assumption.}To enable formal analysis of
the CISS framework, we make the following assumption about the
learning dynamics.

\begin{assumption}[Convergence for each task]
  \label{assumption:convergence}
  For each task $\mathcal{T}_t$, the optimization process converges
  to a locally optimal model parameters $\bm{\theta}_t^*$, such that
  within its neighborhood $\mathcal{N}(\bm{\theta}_t^*)$, the
  magnitude of the loss gradient $\nabla\mathcal{L}_t(\bm{\theta})$
  is bounded (smaller than $\epsilon$) and the Hessian matrix
  $\mathbf{H}_t(\bm{\theta})$ remains positive semi-definite,
  satisfying the second-order conditions for local optimality, i.e.,
  \begin{equation}
    |\nabla\mathcal{L}_t(\bm{\theta})| \leq \epsilon \,,
    \mathbf{H}_t(\bm{\theta}) \succeq 0 \,, \quad \forall \bm{\theta}
    \in \mathcal{N}(\bm{\theta}_t^*) \,.
  \end{equation}
\end{assumption}


Furthermore, it assumes that the magnitude of parameter updates
satisfies $|\Delta_t| < \delta$, $\forall t\leq T$, for some small
constant $\delta > 0$.
{When learning a new task (per
  Assumption~\ref{assumption:convergence}), the model's performance on
  prior tasks can degrade, a phenomenon known as catastrophic
  forgetting. To quantify this, we will define the forgetting rate (see
  below) based on the change in the loss function, which serves as a
  continuous and differentiable {measurement function} for task
  performance and reflects generalization performance when evaluated on
a validation set.}

\begin{definition}[Average Forgetting Rate]\label{definition:forgetting_rate}
  {The forgetting rate for a previously learned task
    $\mathcal{T}_\tau$ (where $\tau<t$) after the model parameters have
    been updated to $\theta_t$ for task $\mathcal{T}_t$, is defined as
    the change in the loss evaluated on the validation set
    $\mathcal{D}^v_\tau$. Formally, it is given by
    $\mathcal{E}_{\tau}(\theta_t)=\mathcal{L}^{val}_{\tau}(\theta_t)-\mathcal{L}^{val}_{\tau}(\theta_{\tau}^*)$,
    where $\mathcal{L}_\tau^{val}$ denotes the loss computed on the
    validation data for task $\mathcal{T}_\tau$. By definition, this
    rate is zero when evaluated at the task's own optimal parameters,
    i.e., $\mathcal{E}_\tau(\theta_\tau^*)=0$. In the context of
    continual learning, the locally optimal parameters $\theta_\tau^*$
    for task $\mathcal{T}_\tau$ may cause forgetting of knowledge from
    previous tasks $\mathcal{T}_{1:\tau-1}$. This effect is measured by
  the Average Forgetting Rate, defined as}
  \begin{equation}
    \bar{\mathcal{E}}_t(\theta_t^*)=\frac{1}{t-1}\sum_{\tau=1}^{t-1}\mathcal{E}_{\tau}(\theta_t^*)\quad
    t\geq 2\,.
  \end{equation}
\end{definition}

{
\subsection{Strict Parameter Isolation in CISS}}
\label{sec:tspi-ciss}

Based on Assumption~\ref{assumption:convergence} and
Definition~\ref{definition:forgetting_rate}, the relationship between
the average forgetting rate for task
$\mathcal{T}_t$ and $\mathcal{T}_{t-1}$ can be obtained as (see
Appendix~\ref{appendix:average_fogetting})
\begin{equation*}\label{eq:forget_rate_rel}
  \begin{split}
    \bar{\mathcal{E}}_t(\theta_t^*)=\frac{1}{t-1}\bigg(&(t-2)\cdot
      \bar{\mathcal{E}}_{{t-1}}(\theta_{t-1}^*)\\
      &+\frac{1}{2}\Delta_t^\intercal\Big(\sum_{\tau=1}^{t-1}\mathrm{H}_{\tau}(\theta_\tau^*)\Big)\Delta_t\\
    &+v^\intercal\Delta_t\bigg)+\mathcal{O}(\delta\cdot\epsilon)
    \,,
  \end{split}
\end{equation*}
where
$v^T=\sum_{\tau=1}^{t-1}(\theta_{t-1}^*-\theta_{\tau}^*)^\intercal\mathbf{H}_{\tau}(\theta_{\tau}^*)$.
From Equation~(\ref{eq:forget_rate_rel}), it is clear that the
average forgetting rate $\bar{\mathcal{E}}_t(\theta_t^*)$ of locally
optimal parameters $\theta_t^*$ for task $\mathcal{T}_t$ is directly
related to the average forgetting rate
$\bar{\mathcal{E}}_{t-1}(\theta_{t-1}^*)$ of locally optimal
parameters $\theta_{t-1}^*$ for task $\mathcal{T}_{t-1}$. With such
relationship, we can obtain the following zero-forgetting condition
(see Appendix~\ref{appendix:average_fogetting}).

\begin{theorem}[Zero-forgetting Condition]
  \label{theorem:null-forgetting}
  For any continuous learning algorithm that satisfies
  Assumption~\ref{assumption:convergence},
  (1) if  $\bar{\mathcal{E}}_{\tau}(\theta_{\tau}^*)=0$,
  $\forall\tau<t$, then
  $\bar{\mathcal{E}}_t(\theta_t^*)=\frac{1}{2(t-1)}\Delta_{t}^{\intercal}\left(\sum_{i=1}^{t-1}\mathbf{H}_{i}(\theta_i^*)\right)\Delta_{t}$,
  and (2)
  $\mathcal{E}_\tau(\theta_t^*)=0,\forall\tau<t$, if and only if
  $\Delta_t^\intercal(\sum_{\tau=1}^{t-1}\mathbf{H}_{\tau}(\theta_{\tau}^*))\Delta_t
  = 0$.
\end{theorem}
Theorem~\ref{theorem:null-forgetting} offers two significant
implications for achieving zero forgetting.  First, even if zero
average forgetting ($\bar{\mathcal{E}}_{\tau}(\theta_{\tau}^*)=0$) is
achieved for all previous tasks $\tau<t$ when evaluated at their
respective optimal parameters, learning a new task $\mathcal{T}_t$
and thereby updating parameters (resulting in $\Delta_t \neq 0$) can
still cause considerable average forgetting. 
Second, Theorem~\ref{theorem:null-forgetting} (second half) provides
a direct mathematical condition for achieving true zero forgetting on
all past tasks (i.e., $\mathcal{E}_\tau(\theta_t^*)=0,
\forall\tau<t$) after the model learns task $\mathcal{T}_t$, i.e.,
the quadratic term
$\Delta_t^\intercal\left(\sum_{\tau=1}^{t-1}\mathbf{H}_{\tau}(\theta_{\tau}^*)\right)\Delta_t$
must be zero. Consequently, any continual learning algorithm aiming
for zero forgetting must be designed to ensure this quadratic term
vanishes. Indeed, existing zero-forgetting methods, such as
orthogonal gradient method~\cite{farajtabar2020orthogonal} and
projected gradient method~\cite{saha2021gradient}, work by satisfying
such a condition (see details in
Appendix~\ref{appendix:other_zero_forgetting_methods_placeholder}).


{Strict parameter isolation (SPI) which is used
  in~\cite{chaSSULSemanticSegmentation2021,yu2025ipseg} is a strategy
  that can satisfy
  the zero forgetting condition in
  Theorem~\ref{theorem:null-forgetting}. By optimizing a unique set of
  parameters for each new task while freezing those for previous tasks,
  the parameter update $\Delta_t$ is guaranteed to be in a subspace
  orthogonal to the parameters of all previous tasks. Consequently, the
  quadratic term $\Delta_t^\intercal(\sum_{\tau=1}^{t-1}
  \mathbf{H}_\tau)\Delta_t$ is always zero, thus ensuring theoretical
zero-forgetting.}


{Although the SPI strategy can theoretically prevent catastrophic
  forgetting by isolating task-specific parameters, its direct
  application in CISS introduces a critical challenge: the problem of
  incomparable outputs. In the SPI framework, each task-specific
  segmentation head is trained independently on a subset of classes.
  Consequently, the output logits from different heads are not
  mutually calibrated; a high score from one head is not directly
  comparable to a score from another trained on a different task.
  This renders the standard approach of applying a simple $argmax$
  operation across all heads' outputs to determine the final class
  for each pixel fundamentally flawed (also can be seen in
  Figure~\ref{fig:framework}(A)).
  This issue stems from the SPI model's inability to determine a
  pixel's task affiliation before classifying it. This problem is
  conceptually analogous to challenges in image-level continual
  learning. A prior study~\cite{kim2022theoretical}, for instance,
  factorizes the image-level prediction into an intra-task prediction
  and a task affiliation probability. Borrowing this formulation for
  our pixel-level problem, the prediction for a pixel $y_p$ can be written as:
  \begin{equation}\label{eq:prob_decomp}
    \begin{split}
      P\left(y_p = c \mid \mathbf{x} \right) = &\; P\left(y_p = c
      \mid \mathbf{x}, \mathcal{T}_t\right)\\
      &\cdot P\left(\mathcal{T}_t \mid \mathbf{x} \right), \quad
      \forall c \in \mathcal{C}_t\,,
    \end{split}
  \end{equation}
  where $p$ is the index of a pixel in the input image $\mathbf{x}$,
  and $c$ is a class learned from task $\mathcal{T}_t$.
  While SPI perfectly preserves the intra-task prediction term
  $P(y_p=c|x,\mathcal{T}_t)$ due to its zero-forgetting nature. There
  is no mechanism to estimate the crucial task affiliation probability
$P(\mathcal{T}_t|x)$, thus hindering effective final prediction.}

\begin{figure*}[t]
  \centering
  \includegraphics[width=\textwidth]{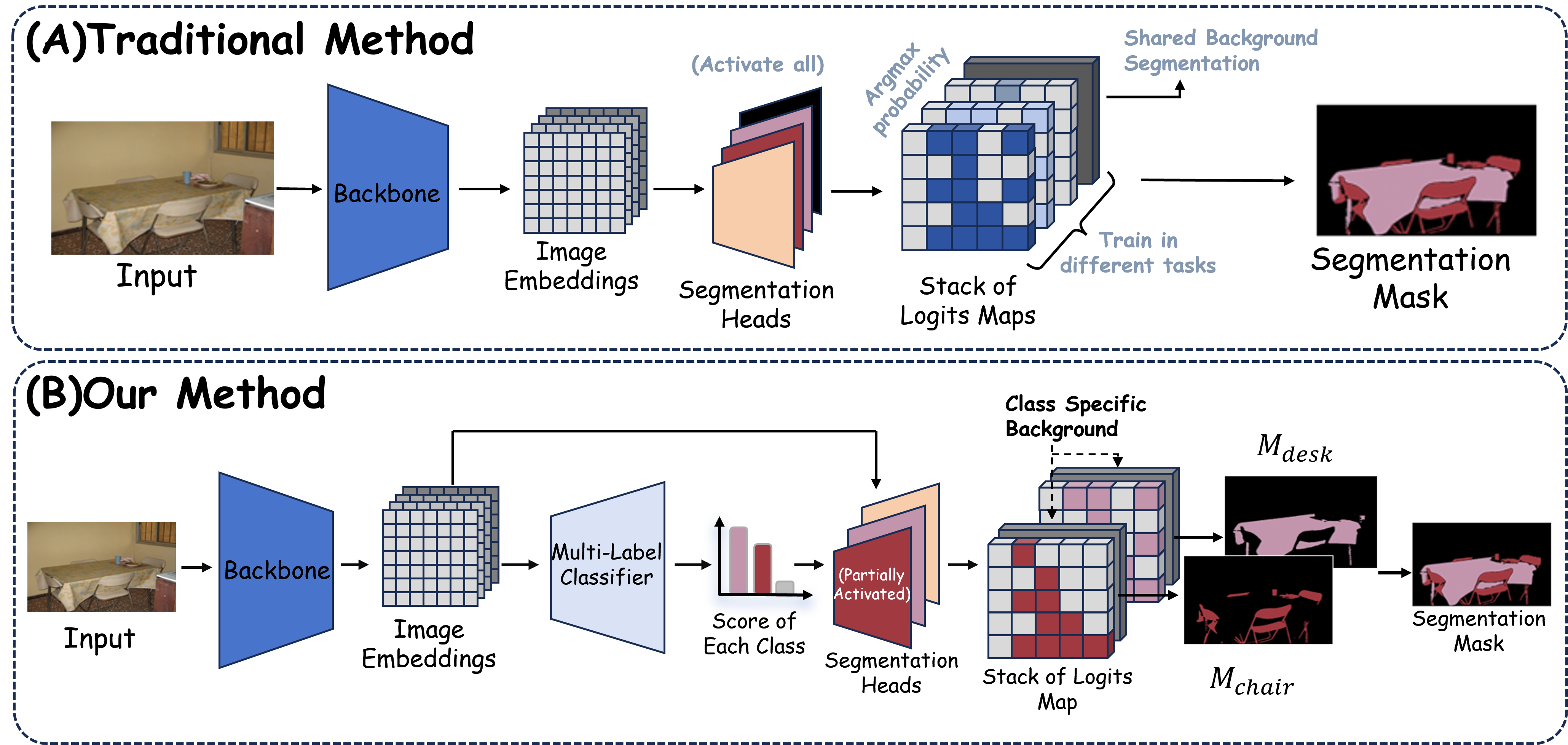}
  \caption{Demonstration of existing typical CISS framework and the
    proposed CogCaS framework. (A) Traditional CISS framework consists
    of a set of segmentation heads which are trained sequentially and
    activated for inference simultaneously. (B) our proposed CogCaS
    restructures the traditional CISS formulation into a dual-phase
    cascade using multi-label classifier and class-specific
    segmentation head. The multi-label classifier determine whether
    each learned class exists in the image, and only segmentation heads
    corresponding to existing classes are activated to produce a
    foreground-background mask. These masks are then fused to obtain
  the final segmentation mask using a mask fusion strategy.}
  \label{fig:framework}
\end{figure*}

\subsection{Class-Incremental Segmentation via Cognitive Cascades}
\label{subsec:method}
{To resolve the issue of incomparable outputs in
  Equation~(\ref{eq:prob_decomp}) and fully leverage the
  zero-forgetting property of SPI, we introduce \textbf{Cog}nitive
  \textbf{Cas}cade Segmentation (CogCaS), a novel framework that
  fundamentally restructures the CISS paradigm. Our approach is
  motivated by the coarse-to-fine strategy employed by human annotators
  who first identify the object categories present in an image before
  delineating their precise boundaries. CogCaS mimics this cognitive
  process by decoupling the problem into two sequential phases. First,
  an image-level classifier which acts as a Task Router determines the
  existence of all learned classes within the input image. Second, only
  the parameter-isolated segmentation heads corresponding to the
  detected classes are activated to perform binary
  foreground-background segmentation. This restructuring can be
  formalized(in Appendix~\ref{appendix:modeling}) by reframing the
  probabilistic decomposition from Equation~(\ref{eq:prob_decomp}) into
a more intuitive, class-centric model:}

\noindent\textbf{Notation Clarification.} Let $Y_p \in
\mathcal{C}_{[1:t]}\cup\{0\}$ denote the random variable representing
the class label of pixel $p$, and let $Z_c \in \{0, 1\}$ be a binary
random variable indicating whether class $c$ is present in image
$\mathbf{x}$ (i.e., $Z_c = 1$ if $\exists\, p : Y_p = c$). The
cascade decomposition is then:
\begin{equation}
  \label{eq:fine_prob_decomp}
  \begin{split}
    P(Y_p = c \mid \mathbf{x}) \propto\;
    &\underbrace{P\bigl(Y_p = c \mid \mathbf{x},\, Z_c = 1
    \bigr)}_{\text{Binary segmentation}}\\
    &\cdot\;
    \underbrace{P\bigl(Z_c = 1 \mid \mathbf{x}\bigr)}_{\text{Class
    existence (router)}},\;\; \forall c \in \mathcal{C}_{[1:t]},
  \end{split}
\end{equation}
where the first term is the conditional probability that pixel $p$
belongs to class $c$ given that class $c$ exists in the image, and
the second term is the existence probability estimated by the
multi-label classifier. In practice, we threshold the existence
probability $P(Z_c = 1|\mathbf{x}) \geq \alpha$ to obtain the
predicted class set $\mathcal{C}_{pred} = \{c : P(Z_c = 1|\mathbf{x})
\geq \alpha\}$.

{The overall pipeline of this framework is depicted in
  Figure~\ref{fig:framework}(B). An input image $\mathbf{x}$ is first
  passed through a task-shared, frozen pretrained feature extractor
  $\Phi$ to generate feature maps $F=\Phi(\mathbf{x})$. Subsequently,
$F$ is processed by the following two cascade phases.}

{\noindent\textbf{Phase I: Image-Level Category Recognition. }This
  phase is designed to infer the presence of each learned class. We use
  a multi-label classifier which operates on the feature maps $F$ to
  yield a class-existence probability $P(Z_c = 1|\mathbf{x})$ for each
  learned class. The set of predicted classes $\mathcal{C}_{pred}$ is
  then identified by applying a threshold $\alpha$ to these
  probabilities, i.e., $\mathcal{C}_{pred} = \{c : P(Z_c =
  1|\mathbf{x}) \geq \alpha\}$. Specifically, we adopt a Sigmoid
  function with a Binary Cross-Entropy loss, which is naturally suited
  for the multi-label classification setting. To enable continual
  learning, each class-specific weight block $\theta_c^{cls}$ is frozen
  after its corresponding training task is complete. This design not
  only achieves Strict Parameter Isolation (SPI) to circumvent
  catastrophic forgetting, but also fundamentally avoids the problem of
  cross-task output comparison by directly learning independent
class-existence probabilities.}

{\noindent\textbf{Phase II: Class-Specific Binary Segmentation. }This
  phase is responsible for estimating the conditional segmentation
  term, $P(Y_p=c|\mathbf{x},Z_c=1)$. For each class $c\in
  \mathcal{C}_{pred}$ identified in Phase I, its dedicated and
  parameter-isolated segmentation head $H_c(\cdot)$ is activated. This
  head takes the feature maps $F$ as input and produces a two-channel
  probability map, $M_c=(m_c^{bg},m_c^{fg})$, representing the
  probability of each pixel belonging to the ``relative background for
  class $c$'' and the ``foreground for class $c$,'' respectively. The
  binary segmentation schema offers key advantages: it simplifies the
  complex multi-class decision boundary into a \emph{single binary
  decision boundary}, leading to faster convergence and higher
fidelity.}

\begin{figure*}[!t]
  \centering
  \includegraphics[width=0.9\textwidth]{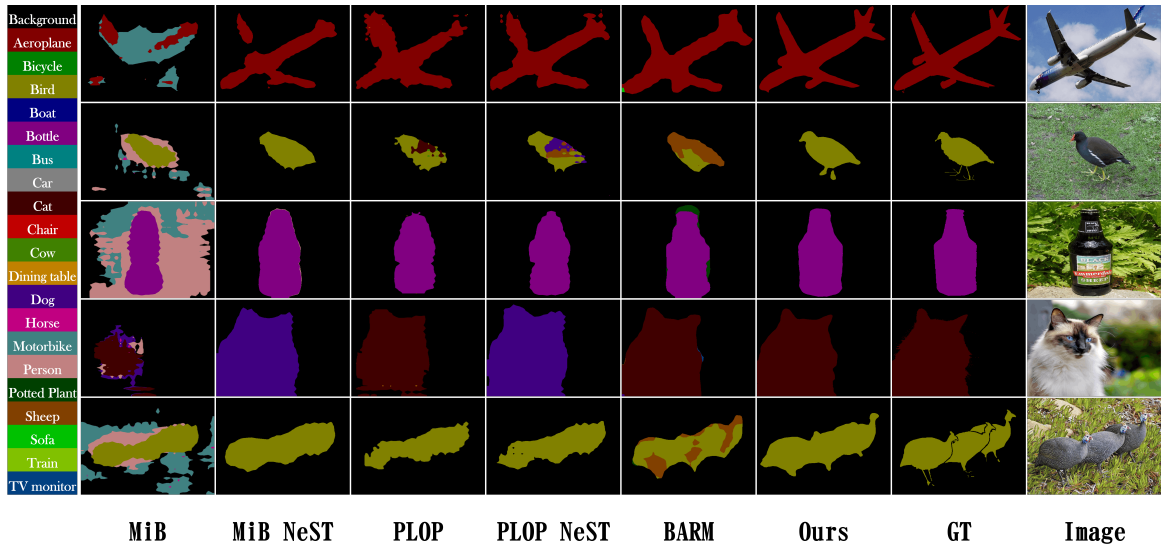}
  \caption{Representative segmentation results from different methods
    after the model learns all tasks in the Pascal VOC 2012 15-1 setting.
  }
  \label{fig:pascal}
\end{figure*}
{\noindent\textbf{Mask Fusion. }Since Phase II yields independent
  binary masks, a fusion step $\mathcal{U}$ is required to integrate
  them into a coherent multi-class segmentation map by resolving
  overlapping predictions. We evaluate several fusion strategies and
analyze their impact in Table~\ref{tab:fusion_strategy}.}

{\noindent\textbf{Precise Model Lifecycle. }To ensure reproducibility
  and clarity, we provide a rigorous description of the complete model
lifecycle:}

{\noindent\emph{Instantiation.} Given an input image $\mathbf{x}$,
  features are extracted as $F=\Phi(\mathbf{x})$, where $\Phi$ is a
  frozen pretrained backbone (e.g., ResNet-101 or Swin-L). When
  learning a new task $\mathcal{T}_t$, we instantiate only new
  parameters: (i) router weights $\theta_c^{cls}$ for each new class
  $c\in\mathcal{C}_t$, and (ii) binary segmentation heads $H_c$ with
  parameters $\theta_c^{seg}$ for each $c\in\mathcal{C}_t$. All
  parameters from previous tasks $\mathcal{C}_{1:t-1}$ remain strictly
frozen.}

{\noindent\emph{Training.} For task $\mathcal{T}_t$, we update only
  the newly instantiated parameters
  $\{\theta_c^{cls},\theta_c^{seg}\}_{c\in\mathcal{C}_t}$. Importantly,
  the multi-label classifier (Phase I) and segmentation heads (Phase
  II) are trained \emph{separately} to avoid interference. The
  training losses for each component are detailed in
  Appendix~\ref{appendix:training}. Gradients are computed only for
  $\{\theta_c^{cls},\theta_c^{seg}\}_{c\in\mathcal{C}_t}$; all other
parameters have zero gradient by design (SPI).}

{\noindent\emph{Inference.} At test time, we first compute existence
  probabilities for all learned classes, then activate only the
  segmenters corresponding to detected classes, and finally fuse their
  binary masks into the final segmentation map. The complete inference
pipeline is detailed in Appendix~\ref{appendix:training}.}

\section{Experiments}
\label{sec:exp}

\subsection{Setup}
\label{subsec:setup}
\noindent\textbf{Datasets.}
Following prior
work~\cite{cermelliModelingBackgroundIncremental2020,chaSSULSemanticSegmentation2021,samrep2021RECALL,yang2023distillation},
the proposed CogCaS was evaluated on two semantic segmentation
datasets PASCAL VOC~\cite{everinghamPascalVisualObject2010} and
ADE20K~\cite{zhouSceneParsingADE20K2017} with different complexity levels.
PASCAL VOC 
contains 20 object classes plus a background class, with 10,582
samples for training and 1,449 samples for validation, while the
large-scale dataset ADE20K presents a more challenging scenario, 
containing 150 foreground classes and one background class, with
20,210 and 2,000 samples for training and validation, respectively.

\noindent\textbf{CISS Settings.}
With each dataset, the widely used $M$-$N$ setting is adopted, with
$M$ being the number of foreground classes in the first task and $N$
the number of new foreground classes in each subsequent task. For
example, in the VOC 10-1 setting, the model first learns to segment
10 classes, then incrementally learns one new class in each subsequent task.


\noindent\textbf{Implementation details.}
The proposed CogCaS was trained using 8 NVIDIA GeForce RTX 4090 GPUs.
We conducted training on models utilizing both
ResNet-101~\cite{heDeepResidualLearning2016} and
Swin-L~\cite{liu2021swin} backbones. The Adam optimizer~\cite{adam}
was employed for training, with each task being trained for 90
epochs. The initial learning rate was set to $1\times10^{-5}$. More
details can be found in the Appendix~\ref{appendix:training}.




\begin{table*}[!t]
  \centering
  \caption[voc-bench]{Comparison with exsiting CISS methods on PASCAL
    VOC 2012 in mIoU (\%). The best results are marked in
    \textbf{bold}. $\circ$: ResNet101 backbone. $\diamond$: Swin-L
    backbone. $\dagger$: unlike other methods, this one is based on
  Mask2Former~\cite{chengPerPixelClassificationNot2021}.}
  \vspace{-2mm}
  \label{tab:voc}
  \resizebox{0.9\textwidth}{!}
  {
    \begin{tabular}{ll|ccc|ccc|ccc}
      \toprule
      & \multirow{2}{*}{Method}
      & \multicolumn{3}{c}{\textbf{10-1} (11 tasks)}
      & \multicolumn{3}{c}{\textbf{15-5} (2 tasks)}
      & \multicolumn{3}{c}{\textbf{15-1} (6 tasks)} \\

      & & 0-10 & 11-20 & all & 0-15 & 16-20 & all & 0-15 & 16-20 & all \\
      \midrule

      \rowcolor{blue!20} &  \textbf{Ours}$\circ$ & \textbf{73.9} &
      \textbf{70.2} & \textbf{72.1} & 75.5 & \textbf{70.3} &
      \textbf{74.1} & 75.5 & \textbf{71.4} & \textbf{74.4} \\
      \rowcolor{blue!20}& \textbf{Joint
      Deeplab-v3}$\circ$~\cite{chen2017rethinking} & 78.4 & 76.4 & 77.4 & 79.8
      & 72.4 & 77.4 & 79.8 & 72.4 & 77.4 \\
      \rowcolor{blue!20}& \textbf{Joint Ours}$\circ$ & 79.3 & 77.6 &
      78.4 & 79.2 & 76.2 & 78.4 & 79.2 & 76.2 & 78.4  \\
      \midrule
      & LwF-MC$\circ$~\cite{samrep2017iCaRL} & 4.7 & 5.9 & 5.0 & 58.1
      & 35.0 & 52.3 & 64.0 & 8.4 & 6.9  \\
      & ILT$\circ$~\cite{michieliContinualSemanticSegmentation2021} &
      7.2 & 3.7 & 5.5 & 67.1 & 39.2 & 60.5 & 8.8 & 8.0 & 8.6  \\
      & MiB$\circ$~\cite{cermelliModelingBackgroundIncremental2020} &
      31.5 & 13.1 & 22.7 &
      71.8 & 43.3 & 64.7 & 46.2 & 22.9 & 40.7  \\
      & MiB+NeST$\circ$~\cite{ECCV2024EarlyPP}& 39.4 & 21.1 & 30.6 &
      75.5 & 48.7 & 69.5 & 60.2 & 29.9 & 53.0 \\
      & PLOP $\circ$~\cite{douillardPLOPLearningForgetting2021}& 44.0
      & 15.5 & 30.5 & 75.4 & 49.6 & 69.3 & 64.1 & 20.1 & 53.1\\
      & PLOP+NeST$\circ$~\cite{ECCV2024EarlyPP} & 47.2 & 16.3 & 32.4
      & 77.6 & 55.8 & 72.4 & 67.2 & 25.7 & 57.3  \\
      & BARM$\circ$~\cite{ECCV2024BackgroundAW} & 72.2 & 49.8 & 61.9
      & 74.9 & 69.5 & 73.6 & \textbf{77.3} & 45.8 & 61.9\\
      & PLOP+LCKD$\circ$~\cite{yang2023distillation} & --- & --- &
      --- & 75.2 & 54.8 & 71.1 & 69.3 & 30.9 & 61.1  \\
      & SSUL$\circ$~\cite{chaSSULSemanticSegmentation2021} & 71.3 &
      46.0 & 59.3 & 77.8
      & 50.1 & 71.2 & \textbf{77.3} & 36.6 & 67.6  \\
      &
      RCIL$\circ$~\cite{zhangRepresentationCompensationNetworks2022}
      & 55.4 & 15.1 & 34.3 &
      \textbf{78.8} & 52.0 & 72.4 & 70.6 & 23.7 & 59.4  \\
      & IDEC$\circ$~\cite{InheritwithDistill} & 70.7 & 46.3 & 59.1 & 78.0 &
      51.8 & 71.8 & 77.0 & 36.5 & 67.3  \\


      \cmidrule{1-11}
      \rowcolor{blue!20} & \textbf{Ours}$\diamond$ & \textbf{76.1} &
      \textbf{75.7} & \textbf{75.9} & 78.3 & \textbf{74.9} &
      \textbf{77.8} & 78.4 & \textbf{72.5} & \textbf{76.9} \\
      \rowcolor{blue!20}& \textbf{Joint
      Deeplab-v3}$\diamond$~\cite{chen2017rethinking} & 81.4 & 78.4 &  79.9 &
      80.8 & 77.3 & 79.9 & 80.8 & 77.3 & 79.9 \\
      \rowcolor{blue!20}& \textbf{Joint Ours}$\diamond$ & 82.7 & 80.9
      & 81.8 & 81.3 & 83.4 & 81.8 & 81.3 & 83.4 & 81.8 \\
      \midrule

      & MicroSeg$\diamond$~\cite{zhangMiningUnseenClasses2022} & 73.5 & 53.0 &
      63.8 & \textbf{81.9} & 54.0 & 75.2 & \textbf{80.5} & 40.8 & 71.0  \\
      &
      MiB$\diamond$~\cite{cermelliModelingBackgroundIncremental2020}
      & 35.7 & 14.8 & 26.7
      & 74.3 & 45.1 & 67.3 & 48.7 & 19.5 & 41.7\\
      & MiB+NeST$\diamond$~\cite{ECCV2024EarlyPP} & 41.3 & 24.1 &
      33.1 & 77.8 & 50.1 & 71.2 & 63.2 & 23.5 & 53.7 \\
      & PLOP$\diamond$~\cite{douillardPLOPLearningForgetting2021} &
      47.2 & 18.4 & 33.5 & 79.2 & 50.2 & 72.3 & 67.6 & 25.2 & 57.6 \\
      & PLOP+NeST$\circ$~\cite{ECCV2024EarlyPP} & 49.2 & 19.8 & 35.2
      & 81.6 & 55.8 & 75.4 & 72.2 & 33.7 & 63.1  \\
      & BARM$\circ$~\cite{ECCV2024BackgroundAW} & 74.2 & 53.8 & 64.4
      & 77.8 & 72.1 & 76.4 & {79.3} & 48.1 & 71.8 \\
      & SSUL$\diamond$~\cite{chaSSULSemanticSegmentation2021} & 74.3
      & 51.0 & 63.2 &
      79.7 & 55.3 & 73.9 & 78.1 & 33.4 & 67.5  \\
      & CoMasTRe$\diamond\dagger$~\cite{CoMasTRe2024} & --- & --- &
      --- & 79.7 & 51.9 & 73.1 & 69.8 & 43.6 & 63.5  \\
      &
      CoMFormer$\diamond\dagger$~\cite{cermelliCoMFormerContinualLearning2023}
      & --- & --- & --- &74.7 & 54.3 & 71.1 & 70.8 & 32.2 & 61.6\\
      \bottomrule
    \end{tabular}
  }
  \vspace{-1.5mm}
\end{table*}

\begin{table}[!tbp]
  \centering
  \caption[ade20k-bench]{Comparison with existing CISS methods on
    ADE20K using Swin-L backbone. $\dagger$: unlike other methods,
    this one is based on Mask2Former~\cite{chengPerPixelClassificationNot2021}.
  }
  \vspace{-2mm}
  \resizebox{\columnwidth}{!}
  {
    \small
    \begin{tabular}{ll|ccc|ccc|ccc}
      \toprule
      &\multirow{2}{*}{\textbf{Method}} &
      \multicolumn{3}{c}{\textbf{100-50} (2 tasks)} &
      \multicolumn{3}{c}{\textbf{100-10} (6 tasks)} &
      \multicolumn{3}{c}{\textbf{100-5} (11 tasks)} \\
      & & \textit{0-100} & \textit{101-150} & \textit{all} &
      \textit{0-100} & \textit{101-150} & \textit{all} &
      \textit{0-100} & \textit{101-150} & \textit{all}\\
      \midrule
      \rowcolor{blue!20}
      & \textbf{Ours}  & {41.2} & \textbf{29.4} & {37.3} &
      \textbf{42.3} & \textbf{25.6} & \textbf{36.8} & {40.1} &
      \textbf{24.7} & \textbf{35.0} \\
      \rowcolor{blue!20}
      & \textbf{Joint Deeplab-v3}~\cite{chen2017rethinking}& 47.2 & 31.8 &
      42.1 & 47.2 & 21.8 & 42.1& 47.2 & 21.8 & 42.1 \\
      \rowcolor{blue!20}
      & \textbf{Joint Ours} & 47.8 & 38.7 & 44.7 & 47.8 & 38.7 & 44.7
      & 47.8 & 38.7 & 44.7 \\
      \midrule
      & MiB~\cite{cermelliModelingBackgroundIncremental2020}& 39.0 & 16.7 & 31.2 & 36.6 &
      9.8 & 27.7 & 34.7 & 4.8 & 24.7 \\
      & MiB+NeST~\cite{ECCV2024EarlyPP}& 38.8 & 23.1 & 33.5 & 38.8 &
      19.1 & 32.2 & 35.2 & 13.6 & 28.1 \\
      & PLOP~\cite{douillardPLOPLearningForgetting2021}& 40.4 & 13.4
      & 31.5 & 39.4 & 12.6 & 30.1 & 36.9 & 6.2 & 26.7 \\
      & PLOP+NeST~\cite{ECCV2024EarlyPP}& 40.8 & 22.8 & 34.8 & 39.4 &
      20.5 & 33.2 & 38.3 & 15.4 & 30.7\\
      & BARM~\cite{ECCV2024BackgroundAW}& 42.0 & 23.0 & 35.7 & 41.1 &
      23.1 & 35.2 & 40.5 & 21.2 & 34.1 \\
      &FALCON~\cite{FALCON}& \textbf{45.9} & 29.1 & \textbf{40.3} &
      41.1 & 23.2 & 35.2 & \textbf{40.8} & 18.9 & 33.5\\
      &
      CoMFormer$\dagger$~\cite{cermelliCoMFormerContinualLearning2023}
      & 44.7 & 26.2 & 38.4 & 40.6 & 15.6 & 32.3 & 39.5 & 13.6 & 30.9\\
      &CoMasTRe$\dagger$~\cite{CoMasTRe2024}& 45.7 & 26.0 & 39.2 &
      42.3 & 18.4 & 34.4 & \textbf{40.8} & 15.8 & 32.6\\


      \bottomrule
      \label{tab:ade}
    \end{tabular}
  }
  \vspace{-5mm}
\end{table}

\begin{table}[!t]
\centering
\caption{Comparison with existing CISS methods under more challenging continual learning settings.
}
\vspace{-2mm}
\label{tab:difficult}
\resizebox{\columnwidth}{!}{%
\begin{tabular}{l|ccc|ccc|ccc}
\toprule
\multirow{2}{*}{\textbf{Method}}
    & \multicolumn{3}{c}{\textbf{VOC 1-1 (20 tasks)}}
    & \multicolumn{3}{c}{\textbf{VOC 2-1 (19 tasks)}}
    & \multicolumn{3}{c}{\textbf{VOC 2-2 (10 tasks)}} \\

    & 0-1 & 2-20 & all & 0-2 & 3-20 & all & 0-2 & 3-20 & all \\
\midrule
\rowcolor{blue!20} Ours & \textbf{79.4} & \textbf{70.1} & \textbf{70.9} & \textbf{76.3} & \textbf{71.2} & \textbf{71.9} & \textbf{75.9} & \textbf{70.8} & \textbf{71.5} \\
\midrule
MiB~\cite{cermelliModelingBackgroundIncremental2020} & 27.3 & 6.4 & 8.3 & 23.6 & 7.9 & 10.14 & 41.1 & 23.4 & 25.9 \\
PLOP~\cite{douillardPLOPLearningForgetting2021} & 25.4 & 4.2 & 6.2 & 19.4 & 6.2 & 8.1 & 39.7 & 22.8 & 25.2 \\
MiB+NeST~\cite{ECCV2024EarlyPP} & 28.1 & 6.8 & 8.7 & 24.5 & 8.1 & 10.4 & 40.4 & 25.8 & 27.8 \\
PLOP+NeST~\cite{ECCV2024EarlyPP} & 32.5 & 4.6 & 7.3 & 20.1 & 7.9 & 10.5 & 38.1 & 23.5 & 25.5 \\
SSUL~\cite{chaSSULSemanticSegmentation2021} & 60.1 & 29.6 & 32.5 & 59.6 & 34.7 & 38.2 & 60.3 & 40.6 & 44.0 \\
IPSeg~\cite{yu2025ipseg} & 61.8 & 30.2 & 33.2 & 60.1 & 32.6 & 36.2 & 64.7 & 49.5 & 51.4 \\

\bottomrule
\end{tabular}%
}
\vspace{-3mm}
\end{table}
\subsection{Main Results}

Experimental results demonstrate the efficacy of the proposed CogCaS
method on the PASCAL VOC 2012 and ADE20K datasets, as presented in
Table~\ref{tab:voc} and Table~\ref{tab:ade} respectively.

Our CogCaS as a non-replay method was compared with basic and
state-of-the-art non-replay baselines. As Table~\ref{tab:voc} shows,
on the PASCAL VOC 2012 dataset, CogCaS exhibited superior performance
across various incremental learning configurations. For example, in
the VOC 10-1 setting (11 tasks), CogCaS achieved the highest mean
Intersection over Union (mIoU) of 70.2\% for new classes (11-20) and
a leading overall mIoU of 72.1\%.
The superiority of CogCaS was more pronounced on the complex ADE20K
dataset 
across all evaluated incremental settings (Table~\ref{tab:ade}).
Figure~\ref{fig:pascal} visually confirms the superior performance of
our method. These results consistently support the efficacy of our
CogCaS in learning new knowledge and preserving old knowledge in both
small-scale and large-scale incremental scenarios.

We also compare with recent large-model-based CISS methods (e.g.,
SAM-based DecoupleCSS with $\sim$632M parameters) in
Appendix~\ref{appendix:large_model}. Our architecture is orthogonal
to backbone choice and offers a favorable balance between performance
and efficiency. Additional experiments on task-shared parameters and
LoRA are provided in Appendix~\ref{appendix:additional_study}.

To further confirm the robustness of our method,
experiments under more challenging conditions were performed in which
tasks are more numerous with fewer classes to be learned within each
class. As Table~\ref{tab:difficult} shows, our CogCaS significantly
outperforms traditional knowledge distillation methods (MiB, PLOP,
and NeST variants) and parameter-isolation strategies (SSUL, IPSeg).
For example, in the VOC 1-1 setting (totally 20 tasks), our method
achieves 70.9\% overall mIoU versus only 7.3\% for PLOP+NeST and
33.2\% for IPSeg. In the VOC 2-2 setting (10 tasks), our method
reaches 71.5\% compared to 27.8\% (MiB+NeST) and 51.4\% (IPSeg). To
verify statistical robustness, we re-evaluated the VOC 1-1 setting
using three independent random seeds, yielding an average mIoU of
71.1\% ($\pm$ 0.8\%), which is highly consistent with the reported
70.9\%. This confirms that the substantial performance gap over IPSeg
(33.2\%) is statistically stable. The low variance stems from the
inherent stability of our SPI design: increasing task count only adds
training cost without affecting previously learned knowledge.

\begin{table}[t]
    \centering
    \caption{Phase~I class-existence detection on the evaluation split (\%). Results are averaged across all settings in the datasets.}
    \label{tab:phase1}
    \small
    \begin{tabular}{lccc}
    \toprule
    \textbf{Dataset} & \textbf{mAP$\uparrow$}  & \textbf{Prec.$\uparrow$} & \textbf{Rec.$\uparrow$} \\
    \midrule
    PASCAL VOC 2012  & 85.12 & 93.47 & 90.25 \\
    ADE20K & 79.84  & 73.86 & 78.03 \\
    \bottomrule
    \end{tabular}
\end{table}
These results clearly demonstrate our CogCaS can well learn new
classes and preserve old knowledge even in a long CISS learning process.



To make the first phase explicit, Table~\ref{tab:phase1} summarises
class-existence detection metrics (mAP, precision, and recall) on
both benchmarks.

\subsection{Ablation Studies}


To assess the practical impact of the classification head during
inference, we conducted comparative experiments with three distinct
model configurations. The first, termed the ``Full Model'', utilizes
the complete model architecture. The second, the ``Segmentation
Only'' version, deactivates the classification head during testing,
relying solely on the segmentation heads learned during training. The
third, the ``Oracle'' version, substitutes the classification head's
output with ground truth labels to isolate the component's error
contribution. By comparing these three settings using standard
semantic segmentation metrics, we can precisely determine the
classification head's actual contribution and significance to our
decoupled segmentation framework at test time.

The ablation experiments of this study in Table~\ref{tab:ablation}
show that there is just small difference in performance between the
``Oracle'' configuration (using real labels) and the ``complete
model'' (for example, both are 70.9\% and 71.2\%), indicating that
the multi-label classifier of the model can effectively handle the 20
categories of VOCs. However, the performance of the ``segmentation
only'' configuration declined, indicating that learning the
relationship between foreground and background remains a challenge
even with the inclusion of near-OOD data during training.

\begin{table}[!tbp]
\centering
\caption{Ablation experiments with respect to the classification head using parameters trained under different task settings}
\vspace{-2mm}
\label{tab:ablation}
\resizebox{\columnwidth}{!}{
\begin{tabular}{l|ccc|ccc|ccc|ccc}
\toprule
\multirow{2}{*}{\textbf{Method}} 
    & \multicolumn{3}{c}{\textbf{VOC 1-1 (20 tasks)}} 
    & \multicolumn{3}{c}{\textbf{VOC 2-1 (19 tasks)}} 
    & \multicolumn{3}{c}{\textbf{VOC 2-2 (10 tasks)}} 
    & \multicolumn{3}{c}{\textbf{ADE 100-5 (11 tasks)}} \\

    & 0-1 & 2-20 & all & 0-2 & 3-20 & all & 0-2 & 3-20 & all & 0-100 & 101-150 & all \\
\midrule

Segmentation Only & 18.2 & 13.1 & 13.5 & 14.3 & 17.2 & 16.7 & 13.8 & 16.5 & 16.1 & 6.7 & 9.5 & 7.6 \\
Full Model & 79.4 & 70.1 & 70.9 & 76.3 & 71.2 & 71.9 & 75.9 & 70.8 & 71.5 & 40.1 & 24.7 & 35.0 \\
Oracle & 80.2 & 70.3 & 71.2 & 77.4 & 71.5 & 72.3 & 77.4 & 71.8 & 72.4 & 48.6 & 31.8 & 43.0 \\
\bottomrule
\end{tabular}}
\vspace{-3mm}
\end{table}

On the ADE20K dataset with more complex categories (150 classes), the
mIoU configured with ``Oracle'' (43.0\%) was significantly better
than that of the ``Full Model'' (35.0\%), revealing that the
classifier encounters challenges when facing a large number of
categories, and its errors have a significant impact on the
segmentation performance.
This cross-dataset contrast indicates that the classifier is
sufficiently accurate for the 20-class VOC, where most detection
errors fall on small or occluded objects that the segmentation head
would also miss; as the class count grows to 150 on ADE20K, the
classifier becomes the performance bottleneck and the Oracle gap
widens accordingly.
These results jointly prove that the classification head is a key
component in this decoupling framework. Especially when there are
many categories, its accuracy is crucial to the final segmentation
effect, and removing the classification head usually leads to performance loss.




\subsection{Sensitivity Study}
\label{exp:mask}

\begin{table}[t]
  \centering
  \caption{Sensitivity analysis of mask fusion strategies on VOC
  across three incremental settings. Results are reported in mIoU (\%).}
  \label{tab:fusion_strategy}
  \resizebox{\columnwidth}{!}{%
    \begin{tabular}{lccc}
      \toprule
      Fusion Strategy   & VOC 10-1 & VOC 15-5 & VOC 15-1 \\
      \midrule
      Logits-based      & 71.4     & 73.3     & 73.4     \\
      Random            & 71.9     & 73.2     & 73.9     \\
      Strict            & 70.9     & 72.8     & 73.0     \\
      Distributed (Ours)& 72.1     & 74.1     & 74.4     \\
      \textbf{Loose}    & \textbf{72.8} & \textbf{74.9} & \textbf{75.2} \\
      \bottomrule
  \end{tabular}}
\end{table}

\textbf{Mask Fusion Strategy Analysis.}
To handle overlapping predictions in segmentation masks, we evaluate
five fusion strategies: (1) \textit{Logits-based}: selects the class
with highest confidence; (2) \textit{Random}: randomly chooses among
overlapping predictions; (3) \textit{Strict}: assigns overlapping
pixels to background; (4) \textit{Distributed}: prioritizes rare
categories to preserve small objects; and (5) \textit{Loose}: accepts
predictions containing the ground truth category. As shown in
Table~\ref{tab:fusion_strategy}, the \textit{Loose} strategy achieves
superior performance across all settings, followed by our
\textit{Distributed} approach. The \textit{Logits-based} and
\textit{Random} strategies show comparable results, while
\textit{Strict} performs worst due to its conservative background
assignment. These results demonstrate the importance of appropriate
overlap handling in incremental segmentation.
We further investigate the effect of task-shared parameters and
class-specific LoRA in
Appendix~\ref{appendix:additional_study}.

\section{Conclusion}
This study provides a theoretical analysis to understand the limitations of existing CISS methods and introduces a novel dual-phase CISS framework, where the segmentation task is split into two disentangled stages. Crucially, the SPI strategy in our design allows the framework to achieve a zero-forgetting rate for knowledge from previous tasks. As a result, the model’s performance is not significantly affected by the number of tasks in CISS, showing strong robustness in long-sequence scenarios where other methods fail. A main limitation is that, although the model’s performance can approach the upper bound of joint training, it requires substantial training resources. Future work includes exploring class-specific fine-tuning of the feature encoder for other imaging modalities and using existing foundational models as segmentation heads in the framework.

\section*{Impact Statement}
This paper presents work whose goal is to advance the field
of Machine Learning. There are many potential societal
consequences of our work, none which we feel must be
specifically highlighted here.

\bibliography{references}
\bibliographystyle{icml2026}

\newpage
\appendix
\onecolumn
\section{Theoretical Proofs}
\label{appendix:theoretical_analysis}

\subsection{Forgetting Rate Analysis}
\label{appendix:average_fogetting}

For the forgetting rate in
Definition~\ref{definition:forgetting_rate}, the average forgetting
rate $\mathcal{E}_{\tau}(\theta)$is defined as the arithmetic average
of the forgetting rates of historical tasks:

$$\mathcal{E}_{\tau}(\theta)=\mathcal{L}^{val}_{\tau}(\theta)-\mathcal{L}^{val}_{\tau}(\theta_{\tau}^*)$$

Applying Taylor expansion to the above equation gives us:

\begin{equation}
  \label{eq:taylor_x}
  \begin{aligned}
    \mathcal{E}_\tau(\theta)=(\theta-\theta_\tau^*)^\intercal{\nabla}\mathcal{L}^{val}_\tau({\theta}_\tau^*)+\frac{1}{2}\left({\theta}-{\theta}_\tau^*\right)^\intercal\mathbf{H}_\tau(\theta_{\tau}^*)({\theta}-{\theta}_\tau^*)+O(\|{\theta}
    -{\theta}_\tau^*\|^3)
  \end{aligned}
\end{equation}

In Equation~(\ref{eq:taylor_x}):
\begin{itemize}
  \item First order
    $(\theta-\theta_\tau^*)^\intercal{\nabla}\mathcal{L}^{val}_\tau({\theta}_\tau^*)$:
    said loss function near the optimal parameters of linear change;
  \item Second order terms
    $\frac{1}{2}\left({\theta}-{\theta}_\tau^*\right)^\intercal\mathbf{H}_\tau(\theta_{\tau}^*)({\theta}-{\theta}_\tau^*)$:
    by Hessian Matrix $\mathbf{H}_\tau(\theta_{\tau}^*)$ said the
    local curvature of loss function;
  \item high order events $O(\|{\theta}-{\theta}_\tau^*\|^3)$: said
    the higher order nonlinear effects.
\end{itemize}

In Assumption~\ref{assumption:convergence}:
$\nabla\mathcal{L}_{\tau}(\theta_{\tau}^{*})\leq \epsilon$, including
$\epsilon$ is a limitless tends to zero (in the optimal parameter
  $\theta_{\tau}^*$ , The gradient of the loss function has gone to
zero), so the first-order term in the Taylor expansion is ignored,
and the forgetting rate is dominated by the second-order term:
\begin{equation}
  \label{eq:taylor}
  \mathcal{E}_\tau(\theta)=\frac{1}{2}\left({\theta}-{\theta}_\tau^*\right)^\intercal\mathbf{H}_\tau(\theta_{\tau}^
  *)({\theta}-{\theta}_\tau^*)+O(\|{\theta}-{\theta}_\tau^*\|^3)
\end{equation}

For two specific tasks $i,j$and the optimal parameters
$\theta_{j}^*$on the tasks $j$(assuming that the parameters vary
  within a range of $\delta$, That's
$\|\theta-\theta_{t}^*\|\leq\delta$) (we've added the range of the
parameter to our assumption). We find that based on
Equation~(\ref{eq:taylor}), this forgetting rate can be simplified to
Equation~(\ref{eq:forgetting_rate_ij}), and the simplified formula is
as follows:

\begin{equation*}
  \begin{aligned}
    \mathcal{E}_{i}(\theta_{j}^*) &= \frac{1}{2}
    (\theta_j^*-\theta_i^*)^\intercal  \mathbf{H}_i(\theta_{i}^*)
    (\theta_j^*-\theta_i^*) + \mathcal{O}(\delta^3) \\
    &= \frac{1}{2} \left(\sum_{\tau=i+1}^j \Delta_\tau
    \right)^\intercal  \mathbf{H}_i(\theta_{i}^*)
    \left(\sum_{\tau=i+1}^j \Delta_\tau \right) + \mathcal{O}(\delta^3) \\
    &= \frac{1}{2} \left(\Delta_{j} + \sum\limits_{\tau=i+1}^{j-1}
    \Delta_{\tau} \right)^\intercal  \mathbf{H}_{i}(\theta_{i}^{*})
    \left(\Delta_{j} + \sum\limits_{\tau=i+1}^{j-1} \Delta_{\tau}
    \right) +\mathcal{O}(\delta^3)\\
    &= \underbrace{\frac{1}{2} \left(\sum_{\tau=i+1}^{j-1}
      \Delta_\tau \right)^\intercal  \mathbf{H}_i(\theta_{i}^*)
      \left(\sum_{\tau=i+1}^{j-1} \Delta_\tau
    \right)}_{\mathcal{E}_i(\theta_{j-1}^*)}  + \frac{1}{2}
    \Delta_j^\intercal  \mathbf{H}_j(\theta_{j}^*)  \Delta_j\\
    &\quad + \frac{1}{2} \left(\sum_{\tau=i+1}^{j-1} \Delta_\tau
    \right)^\intercal  \mathbf{H}_i(\theta_{i}^*)  \Delta_j  +
    \frac{1}{2} \Delta_j^\intercal  \mathbf{H}_i(\theta_{i}^*)
    \left(\sum_{\tau=i+1}^{j-1} \Delta_\tau \right) + \mathcal{O}(\delta^3) \\
    &= \mathcal{E}_i(\theta_{j-1}^*) + \frac{1}{2} \Delta_j^\intercal
    \mathbf{H}_i(\theta_{i}^*)  \Delta_j +
    \left(\sum_{\tau=i+1}^{j-1} \Delta_\tau \right)^\intercal
    \mathbf{H}_i(\theta_{i}^*)  \Delta_j + \mathcal{O}(\delta^3)
  \end{aligned}
\end{equation*}

The Forgetting rate for task $i$ and optimal parameter $\theta_j^*$
can be expressed as Equation~(\ref{eq:forgetting_rate_ij})

\begin{equation}
  \label{eq:forgetting_rate_ij}
  \mathcal{E}_i(\theta_j^*)=\mathcal{E}_i(\theta_{j-1}^*) +
  \frac{1}{2} \Delta_j^\intercal  \mathbf{H}_i(\theta_{i}^*)
  \Delta_j + \left(\sum_{\tau=i+1}^{j-1} \Delta_\tau
  \right)^\intercal  \mathbf{H}_i(\theta_{i}^*)  \Delta_j +
  \mathcal{O}(\delta^3)
\end{equation}

In Equation~(\ref{eq:forgetting_rate_ij}), it consists of several
parts: past forgotten rate: $\mathcal{E}_i(\theta_{j-1}^*)$, namely
after the completion of the task $j-1$ for task $i$ forgotten;
independent effects of the current parameter update:
$\frac{1}{2}\Delta_j^\intercal\mathbf{H}_i(\theta_i^*)\Delta_j$,
directly caused by the parameter update $\Delta_j$of task $j$;
interaction between historical and current updates :
$\left(\sum_{\tau=i+1}^{j-1}\Delta_\tau\right)^\intercal\mathbf{H}_i(\theta_i^*)\Delta_j$,
which reflects the nonlinear superposition effect of the parameter
update sequence, and represents the inner product of the historical
update and the current update. If the two directions are negatively
correlated under the measure of $\mathbf{H}_i(\theta_i^*)$(e.g.,
orthogonal or reverse), forgetting may be alleviated. On the
contrary, if the direction is consistent, the forgetting is aggravated.

Similarly, for the average forgetting rate
$\bar{\mathcal{E}}_t(\theta_t^*)=\frac{1}{t-1}\sum_{\tau=1}^{t-1}\mathcal{E}_\tau(\theta_t^*)$,
we can also simplify it through Equation~(\ref{eq:taylor})
\begin{equation*}
  \begin{aligned}
    \begin{aligned}&\bar{\mathcal{E}_{t}}(\theta_{t}^*)=\frac{1}{t-1}\sum_{\tau=1}^{t-1}\left(\mathcal{E}_\tau(\theta_{t-1}^*)+\frac{1}{2}{\Delta}_t^\intercal\mathbf{H}_\tau(\theta_{\tau}^*){\Delta}_t+\left(\sum_{o=\tau+1}^{t-1}{\Delta}_o\right)^\intercal\mathbf{H}_\tau(\theta_{\tau}^*){\Delta}_t\right)+\mathcal{O}(\delta^3)\\
      &=\frac{1}{t-1}\left(\sum_{\tau=1}^{t-2}\left(\mathcal{E}_\tau(\theta_{t-1}^*)\right)+\frac{1}{2}\sum_{\tau=1}^{t-1}\Delta_t^\intercal\mathbf{H}_\tau(\theta_\tau^*)\Delta_t+\sum_{\tau=1}^{t-1}\left(\sum_{o=\tau+1}^{t-1}{\Delta}_o\right)^\intercal\mathbf{H}_\tau(\theta_{\tau}^*){\Delta}_t\right)+\mathcal{O}(\delta^3)\\
      &=\frac{t-2}{t-1}
      \bar{\mathcal{E}}_{t-1}(\theta_{t-1}^*)+\frac{1}{2(t-1)}\Delta_t^\intercal\left(\sum_{\tau=1}^{t-1}\mathbf{H}_\tau(\theta_{\tau}^*)\right)\Delta_t+\frac{1}{t-1}\left(\underbrace{\sum_{\tau=1}^{t-1}(\theta_{t-1}^*-\theta_\tau^*)^\intercal\mathbf{H}_\tau(\theta_{\tau}^*)}_{v_t^\intercal}\right)\Delta_t\\
      &+\mathcal{O}(\delta^3)\\
      &=\frac{1}{t-1}\left((t-2)
      \bar{\mathcal{E}}_{{t-1}}(\theta_{t-1}^*)+\frac{1}{2}\Delta_t^\intercal(\sum_{\tau=1}^{t-1}\mathbf{H}_\tau(\theta_\tau^*))\Delta_t+v_t^\intercal\Delta_t\right)+\mathcal{O}(\delta^3)
    \end{aligned}
  \end{aligned}
\end{equation*}

In the above simplification, each average forgetting rate has a
specific $v$, which we denote as $v_t$, representing that it belongs
to $\bar{\mathcal{E}}_t(\theta_t^*)$

\begin{equation}
  \label{eq:average_ratev1}
  \bar{\mathcal{E}}_t(\theta_t^*)=\frac{1}{t-1}\left((t-2)\cdot
  \bar{\mathcal{E}}_{{t-1}}(\theta_{t-1}^*)+\frac{1}{2}\Delta_t^\intercal(\sum_{o=1}^{t-1}\mathbf{H}_o^\star)\Delta_t+v_t^\intercal\Delta_t\right)+\mathcal{O}(\delta^3)
\end{equation}

From this equation, we can observe that the average forgetting rate
for task $t$, $\bar{\mathcal{E}}_t(\theta_t^*)$, is influenced by
three principal terms (ignoring the higher-order term
$\mathcal{O}(\delta^3)$) that are averaged:
\begin{itemize}
  \item $(t-2)\cdot\bar{\mathcal{E}}_{t-1}(\theta_{t-1}^*)$ relates
    to the average forgetting rate of the immediately preceding task.
  \item
    $\frac{1}{2}\Delta_t^\intercal(\sum_{o=1}^{t-1}\mathbf{H}_o^\star)\Delta_t$
    captures the impact of the parameter change $\Delta_t$ in
    conjunction with the cumulative Hessian matrices from all prior
    tasks (from $o=1$ to $t-1$).
\end{itemize}

Building on these observations, a further hypothesis can be
formulated: It is proposed that if the average forgetting rate for
each task from $k=2$ up to $k=t-1$ is zero (i.e., the first two items
  in $\bar{\mathcal{E}}_k(\theta_k^*)$ is zero for all
$k\in\{2,3,\ldots,t-1\})$, then for the current $t$-th task, the
component $v^\mathrm{T}\Delta_t$ is also hypothesized to be zero.

We will use mathematical induction to prove this hypothesis. At
first, we suppose the first two terms of
$\bar{\mathcal{E}}_2(\theta_2^*)$ are both 0, then we can get $v_3=0$:
\begin{equation*}
  \begin{aligned}
    \bar{\mathcal{E}}_{3}(\theta_{3}^*)&=\frac{1}{2}\left(1\cdot
    \bar{\mathcal{E}}_{2}(\theta_{2}^*)+\frac{1}{2}\Delta_{3}^{\intercal}(H_{2}(\theta_{2}^*)+H_{1}(\theta_{1}^*))\Delta_{3}+\sum_{t=1}^{2}(\boldsymbol{\theta}_{2}-\boldsymbol{\theta}_{t})^{\intercal}\boldsymbol{H}_{t}^{\star}\boldsymbol{\Delta}_{3}\right)\\
    &=0+\frac{1}{2}\Delta_{3}^{\mathsf{T}}(\frac{1}{2}H_{2}(\theta_{2}^*)+\frac{1}{2}H_{1}(\theta_{1}^*))\Delta_{3}+\frac{1}{2}\underbrace{\Delta_{2}^{\mathsf{T}}H_{1}^{1}}_{=\mathbf{0}}\Delta_{3}+\frac{1}{2}\underbrace{(\boldsymbol{\theta}_{2}^{\mathsf{T}}-\boldsymbol{\theta}_{2})}_{=\mathbf{0}}H_{t}^{\star}\Delta_{3}
  \end{aligned}\,.
\end{equation*}

Then, we suppose the first tow terms of
$\bar{\mathcal{E}}_{t-1}(\theta_{t-1}^*)$ is zero, we find that:
\begin{equation*}
  \begin{aligned}
    v_{t}-v_{t-1}&=\sum_{o=1}^{t-1}(\theta_{t-1}^*-\theta_{o}^*)^\intercal\mathbf{H}_{o}\left(\theta_{o}^*\right)-\sum_{o=1}^{t-2}(\theta_{t-2}^*-\theta_{o}^*)^\intercal\mathbf{H}_{o}\left(\theta_{o}^*\right)\\
    &=\sum_{o=1}^{t-2}(\theta_{t-1}^*-\theta_{o}^*)^\intercal\mathbf{H}_{o}\left(\theta_{o}^*\right)-\sum_{o=1}^{t-2}(\theta_{t-2}^*-\theta_{o}^*)^\intercal\mathbf{H}_{o}\left(\theta_{o}^*\right)\\
    &=\sum_{o=1}^{t-2}\left(\theta_{t-1}^*-\theta_{o}^*-\theta_{t-2}^*+\theta_{o}^*
    \right)^\intercal\mathbf{H}_{o}\left(\theta_{o}^*\right)
  \end{aligned}
\end{equation*}

Since in $\bar{\mathcal{E}}_{t-1}(\theta_{t-1}^*)$,
$\frac{1}{2}\Delta_{t-1}^\intercal\sum_{o=1}^{t-2}H_{o}\left(\boldsymbol{\theta}_{o}^*\right)\Delta_{t-1}=0$,
So we find that: $v_{t}-v_{t-1}=0$.

Also we can get the conclusion that:
\begin{equation}
  \label{eq:vt}
  v_t-v_{t-1}=\Delta_{t-1}^\intercal\sum_{\tau=1}^{t-2}\mathbf{H}_\tau(\theta_\tau^*)
\end{equation}

We can say that if $\bar{\mathcal{E}}_{\tau}(\theta_{\tau}^*)=0$,
$\forall\tau<t$, then
$\bar{\mathcal{E}}_t(\theta_t^*)=\frac{1}{2(t-1)}\Delta_{t}^{\intercal}\left(\sum_{i=1}^{t-1}\mathbf{H}_{i}(\theta_i^*)\right)\Delta_{t}$,
which is show in Theorem~\ref{theorem:null-forgetting} (1).

In Theorem~\ref{theorem:null-forgetting} (2), we need to proof the
statement:
$\mathcal{E}_\tau(\theta_t^*)=0,\forall\tau<t\Longleftrightarrow\Delta_t^\intercal\left(\sum_{\tau^{\prime}=1}^{t-1}\mathbf{H}_{\tau^{\prime}}(\theta_{\tau^{\prime}}^*)\right)\Delta_t=0\,.$
It is clear that when $\mathcal{E}_\tau(\theta_t^*)=0,\forall\tau<t$,
then $\bar{\mathcal{E}}_\tau(\theta_\tau^*)=0, \forall\tau<t$, using
conclusion in Theorem~\ref{theorem:null-forgetting}(1), we can get
the result:
$\Delta_t^\intercal\left(\sum_{\tau^{\prime}=1}^{t-1}\mathbf{H}_{\tau^{\prime}}(\theta_{\tau^{\prime}}^*)\right)\Delta_t=0$.

When
$\Delta_t^\intercal\left(\sum_{\tau^{\prime}=1}^{t-1}\mathbf{H}_{\tau^{\prime}}(\theta_{\tau^{\prime}}^*)\right)\Delta_t=0$,
using Assumption~\ref{assumption:convergence}, all Hessian matrix
$\mathbf{H}_\tau(\theta_\tau^*)$ is positive semi-definite, then we
have:
$\Delta_\tau^\intercal(\sum_{o=1}^{\tau-1}\mathbf{H}_o(\theta_o^*))=0,
\forall\tau<t$, using Equation~(\ref{eq:vt}), we have: $v_2\leq
v_3\cdots \leq v_t\leq v_{t+1}$. In definition,
$\bar{\mathcal{E}}_2(\theta_2^*)$ is zero, and
$\bar{\mathcal{E}}_3(\theta_3^*)=\frac{1}{8}\Delta_3^\intercal(\sum_{o=1}^2\mathbf{H}_o{\theta_o^*})\Delta_3=0$,
we have $\bar{\mathcal{E}}_t(\theta_t^*)=0$ which can say that
$\mathcal{E}_\tau(\theta_t^*)=0,\forall\tau<t$

\subsection{Other Zero Forgetting strategy}
\label{appendix:other_zero_forgetting_methods_placeholder}

\subsubsection{Orthogonal Gradient Method}

We follow the setting of the paper~\cite{farajtabar2020orthogonal},
where $\mathcal{L}$ is a non-negative loss function(CE or BCE loss).
We also assume that the symbol $f_\theta$ represents the parameter
$\theta$ used by the model ($f_\theta^c(x)$ means one of the output's
channel which is about class $c$) and $N_t$ represents the total
amount of data used in the tasks from $1$ to $t$.

In Orthogonal Gradient Method~\cite{farajtabar2020orthogonal}, they
want to address catastrophic forgetting in continual learning by
keeping the updates for a new task orthogonal to the gradient
directions associated with previous tasks' predictions. Formally,
equal to Equation~(\ref{eq:ogm}).
\begin{equation}
  \label{eq:ogm}
  \langle\Delta_t^i,\nabla_{\theta_{t-1}^*}f_{\theta_{t-1}^*}^c(x_{\tau})\rangle=0\quad\forall
  c\in\mathcal{C}_{[1:t-1]},{x}_{\tau} \in D_\tau,\tau<t\,,
\end{equation}
where $\Delta_t^i$ denotes the the $i$-th step for update.

Due to the previous studies~\cite{hessian}, the hessian matrix of the
loss can be decomposed as two other matrices: the outer-product
Hessian and the functional Hessian, and at the optimum parameter for
loss function, the functional Hessian is
negligible~\cite{hessian2021}. So we can write the approximation of
the Hessian matrix under the optimal parameters in task $\tau$:
\begin{equation*}
  \mathbf{H}_\tau(\theta_\tau^*)=\frac{1}{N_\tau}\sum_{i=1}^{N_\tau}\nabla_{\theta_\tau^*}f_{\theta_\tau^*}(x_i)(\nabla_f^2\mathcal{L}_\tau(x_i,y_i))\nabla_{\theta_\tau^*}f(x_i)^\intercal
\end{equation*}

When the parameter update follows Equation~(\ref{eq:ogm}), it
essentially satisfies the condition that the binomial is 0 mentioned
in our zero forgetting condition in Theorem~\ref{theorem:null-forgetting}.



\subsection{Why Parameters Isolation is Zero Forgetting}
\label{appendix:parameters_isolation}
We find that when the SPI strategy is used, the sum of the historical
Hessian
matrix($\sum_{\tau=1}^{t-1}\mathbf{H}_{\tau}(\theta_{\tau}^*)$) is
similar to a semi-positive definite block diagonal matrix to its
parameter subspace:

\begin{equation}
  \label{eq:sum_hessian}
  \sum_{\tau=1}^{t-1}\mathbf{H}_{\tau}(\theta_{\tau}^*) =
  \begin{pmatrix}
    \mathbf{H}_{1}(\theta_{1}^*) & \mathbf{0} & \cdots & \mathbf{0}
    &\cdots&\mathbf{0}\\
    \mathbf{0} & \mathbf{H}_{2}(\theta_{2}^*) & \cdots & \mathbf{0}
    &\cdots&\mathbf{0}\\
    \vdots & \vdots & \ddots & \vdots &\cdots&\mathbf{0}\\
    \mathbf{0} & \mathbf{0} & \cdots &
    \mathbf{H}_{t-1}(\theta_{t-1}^*)&\cdots&\mathbf{0}\\
    \vdots&\vdots&\vdots&\vdots&\vdots&\vdots\\
    \mathbf{0} & \mathbf{0} & \mathbf{0} & \mathbf{0}& \mathbf{0}& \mathbf{0}
  \end{pmatrix}\,.
\end{equation}
Where:
\begin{itemize}
  \item According to Assumption~\ref{assumption:convergence}, Each
    $\mathbf{H}_\tau(\theta_\tau^*)$ on the diagonal is a square
    matrix representing the Hessian for task $\tau$ and is a
    semi-positive matrix. The dimensions of
    $\mathbf{H}_\tau(\theta_\tau^*)$ correspond to the number of new
    parameters introduced for task $\tau$.
  \item The $\mathcal{0}$ symbols represent zero, indicating that the
    Hessian components for parameters of different tasks are
    decoupled due to SPI strategy.
\end{itemize}

Just as expressed in Equation~(\ref{eq:sum_hessian}), the sum of the
historical tasks' Hessian matrices
$\sum_{\tau=1}^{t-1}\mathbf{H}_{\tau}(\theta_{\tau}^*)$ and the
parameter update for new task $\Delta_t$ are not in the same
subspace, and the product between them must be 0 which is the zero
forgetting condition in Theorem~\ref{theorem:null-forgetting}.

\subsection{Cascade Modeling}
\label{appendix:modeling}

\paragraph{Variable Definitions.}
We first clarify the random variables used throughout this derivation:
\begin{itemize}
  \item $Y_p \in \mathcal{C}_{[1:t]} \cup \{0\}$: the random variable
    representing the class label assigned to pixel $p$, where $0$
    denotes background.
  \item $Z_c \in \{0, 1\}$: a binary random variable indicating
    whether class $c$ is present in image $\mathbf{x}$. Formally,
    $Z_c = 1 \Leftrightarrow \exists\, p : Y_p = c$.
\end{itemize}

We start from the task-level factorization introduced in
Equation~(\ref{eq:prob_decomp}):
$$P(Y_p = c \mid \mathbf{x}) = P\bigl(Y_p = c \mid \mathbf{x},
\mathcal{T}_t\bigr)\, P\bigl(\mathcal{T}_t \mid \mathbf{x}\bigr),
\quad c \in \mathcal{C}_t,$$
where $\mathcal{T}_t$ denotes the task that first introduced class
$c$. Under SPI, classes are disjoint across tasks:
$\mathcal{C}_\tau \cap \mathcal{C}_t = \varnothing$ for all $\tau \neq t$.

\begin{lemma}[Task Prior Decomposition]\label{lem:task_prior_decomp}
  For any image $\mathbf{x}$,
  \begin{equation}
    \label{eq:task_to_class_prior}
    P(\mathcal{T}_t \mid \mathbf{x}) = P\left(\bigvee_{c \in
      \mathcal{C}_t} Z_c = 1 \;\bigg|\; \mathbf{x}\right) \approx
    \sum_{c \in \mathcal{C}_t} P(Z_c = 1 \mid \mathbf{x}),
  \end{equation}
  where $P(Z_c = 1 \mid \mathbf{x})$ is the probability that class
  $c$ exists somewhere in image $\mathbf{x}$, and the approximation
  holds when class co-occurrence within a single task is rare.
\end{lemma}

\begin{proof}
  Since classes in $\mathcal{C}_t$ are mutually exclusive and task
  $\mathcal{T}_t$ is active if and only if at least one class from
  $\mathcal{C}_t$ appears in $\mathbf{x}$, we have
  $$P(\mathcal{T}_t \mid \mathbf{x}) = P\left(\bigvee_{c \in
    \mathcal{C}_t} Z_c = 1 \;\bigg|\; \mathbf{x}\right) \approx
  \sum_{c \in \mathcal{C}_t} P(Z_c = 1 \mid \mathbf{x})$$
  by the inclusion-exclusion principle (approximated when inter-class
  overlap is negligible).
\end{proof}

\paragraph{From Task-Level to Class-Level Factorization.}
The task-level factorization necessarily reduces to a simpler
class-level form that eliminates the problematic task prior
$P(\mathcal{T}_t \mid \mathbf{x})$.

\begin{theorem}[Cascade Factorization]
  \label{thm:cascade}
  For any pixel $p$, class $c$, and image $\mathbf{x}$:
  \begin{equation}
    \label{eq:cascade}
    P(Y_p = c \mid \mathbf{x}) = P\bigl(Y_p = c \mid \mathbf{x}, Z_c =
    1\bigr) \cdot P(Z_c = 1 \mid \mathbf{x}).
  \end{equation}
\end{theorem}

\begin{proof}
  We derive this using the law of total probability over the class
  existence variable $Z_c$:
  \begin{align*}
    P(Y_p = c \mid \mathbf{x}) &= P(Y_p = c \mid \mathbf{x}, Z_c = 1)
    \cdot P(Z_c = 1 \mid \mathbf{x}) \\
    &\quad + P(Y_p = c \mid \mathbf{x}, Z_c = 0) \cdot P(Z_c = 0 \mid \mathbf{x}).
  \end{align*}
  The key insight: if class $c$ is not present in the image ($Z_c =
  0$), then no pixel can be labeled as $c$. Therefore, $P(Y_p = c
  \mid \mathbf{x}, Z_c = 0) = 0$. This eliminates the second term,
  yielding:
  $$P(Y_p = c \mid \mathbf{x}) = P(Y_p = c \mid \mathbf{x}, Z_c = 1)
  \cdot P(Z_c = 1 \mid \mathbf{x}).$$
\end{proof}

\paragraph{Normalization over All Classes.}
The final pixel-level prediction requires normalization over all
candidate classes. Let $\mathcal{C}_{\text{pred}} = \{c : P(Z_c = 1
\mid \mathbf{x}) \geq \alpha\}$ be the set of predicted classes after
thresholding. The normalized probability is:
\begin{equation}
  \label{eq:cascade_normalized}
  P(Y_p = c \mid \mathbf{x}) = \frac{P(Y_p = c \mid \mathbf{x}, Z_c =
    1) \cdot P(Z_c = 1 \mid \mathbf{x})}{\sum_{c' \in
      \mathcal{C}_{\text{pred}} \cup \{0\}} P(Y_p = c' \mid
    \mathbf{x}, Z_{c'} = 1) \cdot P(Z_{c'} = 1 \mid \mathbf{x})},
\end{equation}
where $c' = 0$ corresponds to the background class with $P(Z_0 = 1
\mid \mathbf{x}) = 1$ (background is always present).

\paragraph{Implications for CogCaS Architecture.}
Theorem~\ref{thm:cascade} proves that under SPI, the cascade
factorization in Equation~(\ref{eq:cascade}) is the unique
probabilistically consistent decomposition. This directly motivates
our two-phase CogCaS design:
\begin{itemize}
  \item \textbf{Phase I}: Multi-label classifier estimates class
    existence probabilities $P(Z_c = 1 \mid \mathbf{x})$
  \item \textbf{Phase II}: Binary segmentation heads model
    conditional segmentation $P(Y_p = c \mid \mathbf{x}, Z_c = 1)$
\end{itemize}

This decomposition eliminates the ill-defined task prior, enables
complete parameter isolation for zero-forgetting, and aligns with
optimal Bayesian factorization principles.

\section{Implementation Details}
\label{appendix:implementation}

\subsection{Training and Inference}
\label{appendix:training}

\paragraph{CISS Training.}
For every new task $T_t$, we append one multi-label classifier head
and one class-specific binary segmentation head (Deeplab-v3's ASPP
module + $1 \times 1$ conv) for each unseen class $c \in
\mathcal{C}_t$, while all previous weights are frozen (*Strict
Parameter Isolation*).

Furthermore, during the training stage, we manually construct
near-out-of-distribution (near-OOD) data based on the data available
for the current task to enhance the model's robustness. As
illustrated in the Figure~\ref{fig:nearood}, this process is divided
into two distinct stages: Phase I and Phase II. It is important to
note that our model is trained separately in these two phases.
\begin{figure}[ht]
  \centering
  \includegraphics[width=0.5\linewidth]{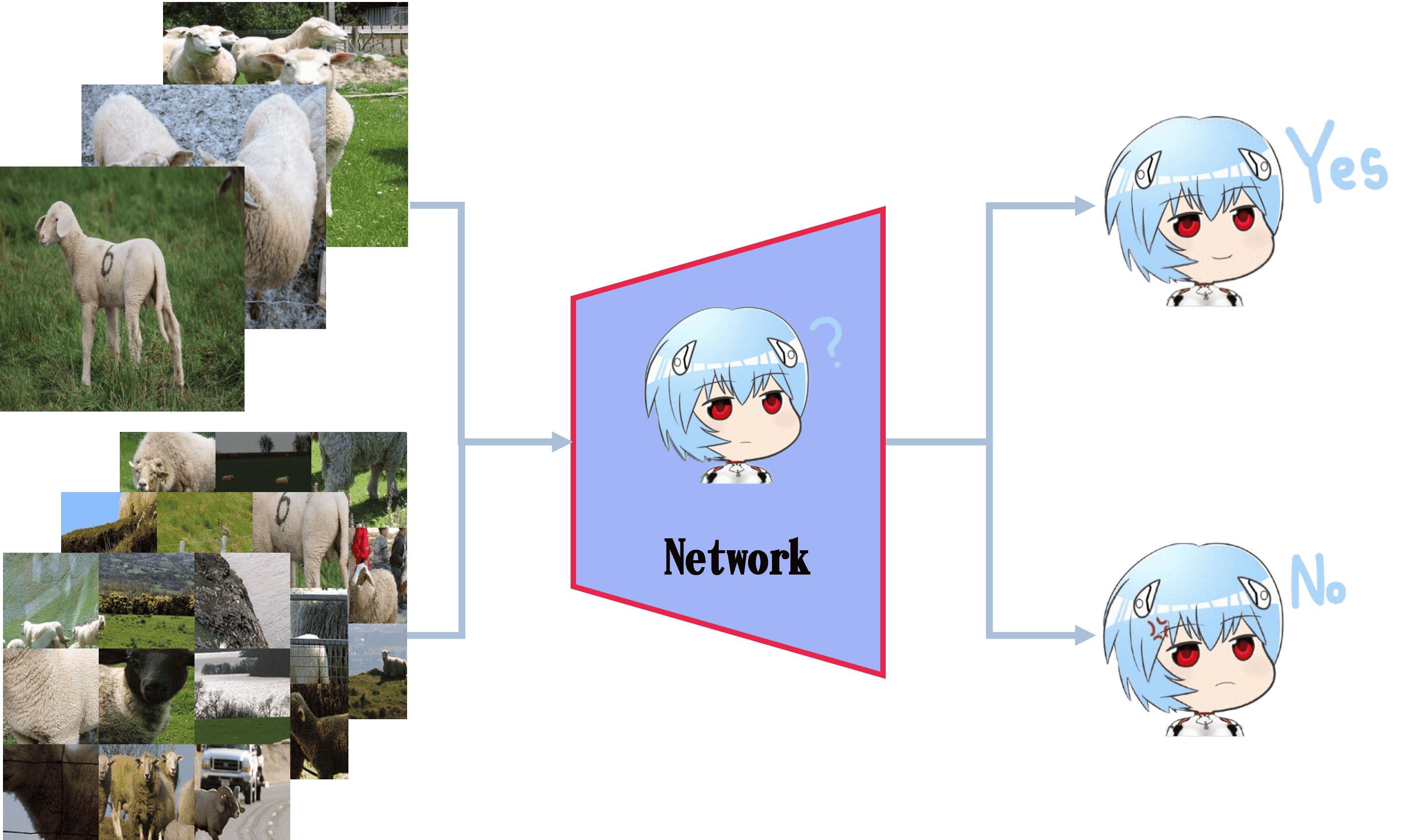}
  \caption{The figure illustrate near-ood samples and process.}
  \label{fig:nearood}

\end{figure}
\vspace{-2pt}
\paragraph{Joint Training.}
In the joint training setting, we first train the classifier in Phase
I. After the training for Phase I is complete, because we have access
to the full dataset, we then select corresponding out-of-distribution
data for each class-specific segmentation head in Phase II. This data
is chosen from the dataset at a 1:1 ratio, and we ensure that for a
specific class, its corresponding OOD samples do not contain any
information about that class in any single pixel.

\paragraph{Parameter Overhead.}  Each additional class contributes
\mbox{$\approx\!1.37$\,M} parameters (\mbox{$5.24$ MB} at FP32),
corresponding to
\mbox{3.1 \%} of a ResNet-101 backbone (44.5 M params) and
\mbox{4.9 \%} of a Swin-T backbone (28 M params); even for the larger
Swin-B (88 M params) the overhead per class is merely \mbox{1.6 \%}.

\paragraph{Loss Functions.}  Classification and segmentation are optimised
separately and then summed:
\vspace{-1.5pt}
\begin{equation*}
  \begin{alignedat}{2}
    \mathcal{L}_{\text{cls}}
    &= \tfrac{1}{|\mathcal{B}|}
    \sum_{x\in\mathcal{B}}
    \sum_{c\in\mathcal{C}_{1:t}}
    \operatorname{BCE}\!\bigl(P(c\mid x),\,y_c\bigr) \\[0.6ex]
    \mathcal{L}_{\text{seg}}
    &= \tfrac{1}{|\mathcal{B}|}
    \sum_{x\in\mathcal{B}}
    \sum_{c\in\mathcal{C}_t}
    \Bigl[
      \alpha\,\operatorname{Focal}(M_c,\hat M_c)
      + \beta\,\operatorname{Dice}(M_c,\hat M_c)
    \Bigr] \\[0.9ex]
    \operatorname{Focal}(M_c, \hat{M}_c)
    &= -\tfrac{1}{N} \sum_{i=1}^{N} \Big[
      \alpha\, M_{c,i} (1 - \hat{M}_{c,i})^\gamma \log(\hat{M}_{c,i})
      + (1-\alpha)(1 - M_{c,i}) (\hat{M}_{c,i})^\gamma \log(1 - \hat{M}_{c,i})
    \Big] \\[0.9ex]
    \operatorname{Dice}(M,\hat M)
    &= \frac{2\sum_i M_i\hat M_i+\varepsilon}{\sum_i M_i+\sum_i\hat
    M_i+\varepsilon},
    \quad \varepsilon=10^{-6}
  \end{alignedat}
\end{equation*}




\noindent\textbf{Notation.}
$\mathcal{B}$: mini-batch;
$\mathcal{C}_{1:t}$: all classes learned up to task $t$;
$\mathcal{C}_{t}$: classes introduced at task $t$;
$P(c\mid x)$: predicted presence probability for class $c$;
$y_c\!\in\!\{0,1\}$: image-level label;
$M_c,\hat M_c$: predicted / ground-truth masks;
$i$: spatial index.

The full objective is
$\mathcal{L}=\mathcal{L}_{\text{cls}}+\lambda\,\mathcal{L}_{\text{seg}}$
with $\lambda{=}1$.

\textbf{Optimisation Schedule.}
Epoch 1 trains only the new classifier heads; the remaining epochs finetune
both classifier and segmentation heads.
SGD (momentum 0.9, weight-decay $10^{-4}$), batch size 20,
initial LR $5\times10^{-3}$ with cosine decay is used.

\paragraph{Inference.}%
\begin{algorithm}[H]
  \caption{Inference Pipeline (per image)}
  \begin{algorithmic}[1]
    \REQUIRE backbone $\Phi$, multi-label head $G$, binary heads $\{H_c\}$
    \STATE $F \leftarrow \Phi(x)$ \hfill \COMMENT{shared feature map}
    \STATE $\mathbf{p} \leftarrow G(F)$ \hfill
    \COMMENT{class-presence probabilities}
    \STATE $\mathcal{C}_{\text{pred}} \leftarrow \{\,c \mid
    \mathbf{p}[c] > 0.5\,\}$ \hfill \COMMENT{or top-$k$}
    \FORALL{$c \in \mathcal{C}_{\text{pred}}$}
    \STATE $M_c \leftarrow H_c(F)$ \hfill \COMMENT{binary mask for class $c$}
    \ENDFOR
    \STATE $M_{\text{final}} \leftarrow
    \mathcal{U}\!\bigl(\{M_c\}_{c\in\mathcal{C}_{\text{pred}}}\bigr)$
    \hfill \COMMENT{mask fusion}
    \STATE \textbf{return} $M_{\text{final}}$
  \end{algorithmic}
\end{algorithm}
Because only $|\mathcal{C}_{\text{pred}}|$ segmentation heads are activated,
inference cost scales with the number of present classes rather than the total
number of learned classes.

\section{Additional Experiments}
\label{appendix:experiments}

\subsection{Comparison with Large-Model-Based Methods}
\label{appendix:large_model}

We compare CogCaS with recent large-model-based CISS methods in
Table~\ref{tab:large_model}. SAM-based methods (e.g., DecoupleCSS) leverage a
foundation model with $\sim$632M parameters pre-trained on 11M images,
representing a fundamentally different resource regime compared to our standard
backbones (ResNet-101: 44.5M; Swin-L: 197M). Our ``Cognitive Cascade''
architecture and theoretical zero-forgetting guarantee are orthogonal to
backbone choice and can be readily integrated with SAM or similar models.
CogCaS offers a favorable balance between performance and efficiency,
particularly for resource-constrained deployment scenarios.

\begin{table}[t]
    \centering
    \caption{Comparison with large-model-based CISS methods. $\star$: SAM-based ($\sim$632M params). $\ddagger$: Mask2Former-based.}
    \label{tab:large_model}
    \small
    \begin{tabular}{llccc}
    \toprule
    \textbf{Method} & \textbf{Setting} & \textbf{Old} & \textbf{New} & \textbf{All} \\
    \midrule
    \multicolumn{5}{c}{\textit{PASCAL VOC 2012}} \\
    \midrule
    DecoupleCSS$\star$ & 19-1 & 82.9 & 83.7 & 84.0 \\
    DecoupleCSS$\star$ & 15-1 & 83.8 & 82.1 & 83.4 \\
    \rowcolor{blue!20} Ours$\circ$ & 15-1 & 75.5 & 71.4 & 74.4 \\
    \midrule
    \multicolumn{5}{c}{\textit{ADE20K}} \\
    \midrule
    DecoupleCSS$\star$ & 100-10 & 58.2 & 52.0 & 56.9 \\
    DecoupleCSS$\star$ & 100-5 & 57.5 & 55.6 & 56.9 \\
    CIT-M2Former$\ddagger$ & 100-10 & 56.9 & 38.2 & 50.7 \\
    CIT-M2Former$\ddagger$ & 100-5 & 55.5 & 32.2 & 47.7 \\
    SimCIS$\ddagger$ & 100-10 & 49.7 & 27.4 & 42.3 \\
    SimCIS$\ddagger$ & 100-5 & 46.7 & 22.8 & 38.7 \\
    \rowcolor{blue!20} Ours$\diamond$ & 100-10 & 49.2 & 44.0 & 47.4 \\
    \rowcolor{blue!20} Ours$\diamond$ & 100-5 & 49.5 & 43.6 & 47.3 \\
    \bottomrule
    \end{tabular}
\end{table}

\subsection{Effect of Task-Shared Parameters}
\label{appendix:additional_study}

\begin{wrapfigure}{r}{0.5\textwidth}
  \centering
  \vspace{-12pt}
  \includegraphics[width=0.48\textwidth]{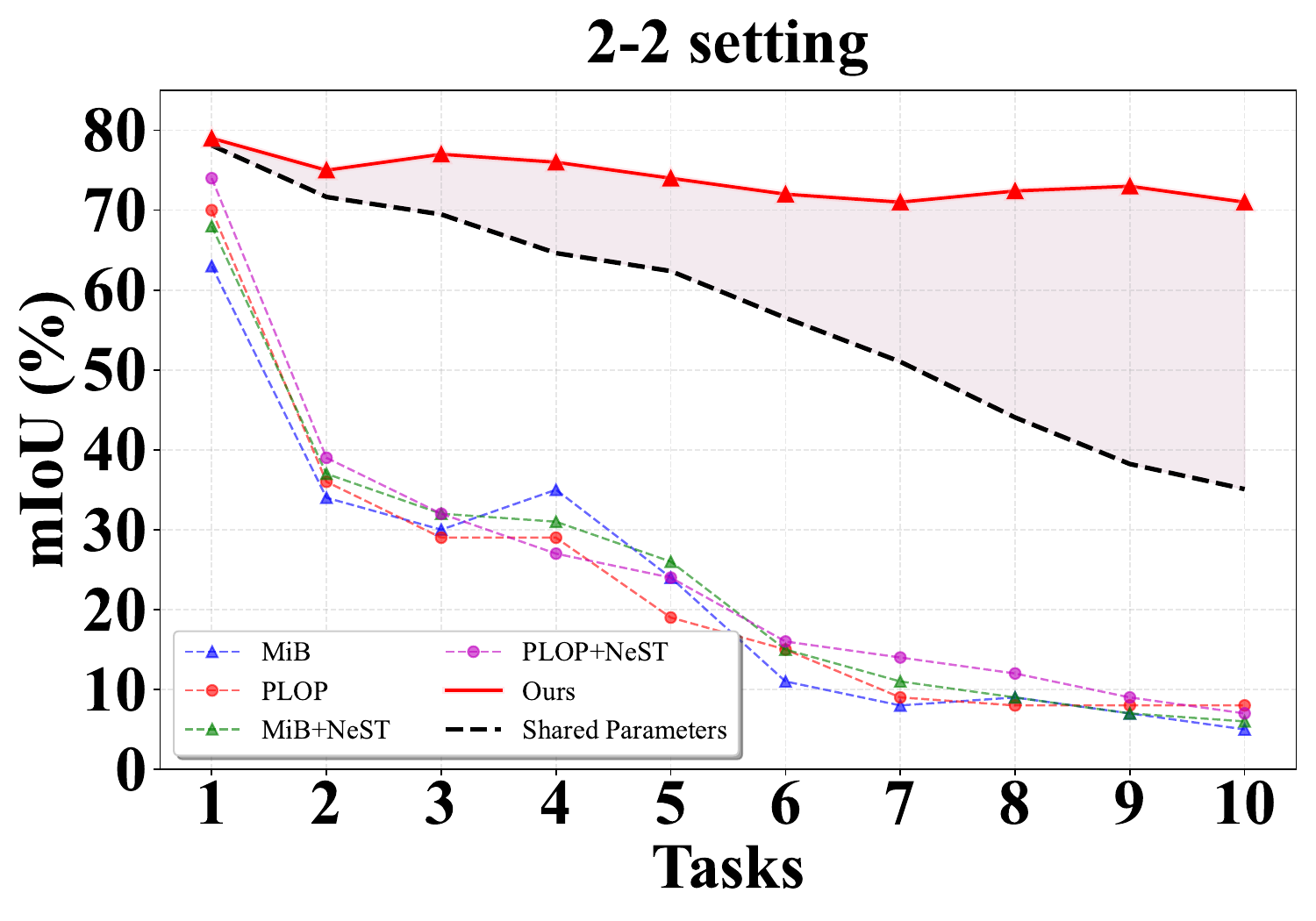}
  \caption{Impact on model performance when the task-shared parameters
  are set, evaluated on the VOC 2-2 setting.}
  \label{fig:addtional}
  \vspace{-10pt}
\end{wrapfigure}
Due to limited resources and time, our further investigations focused
on the challenging 2-2 settings on the VOC dataset. To further
validate our method, we unfroze the Encoder's final bottleneck layer,
making it task-shared and trainable, as shown in
Figure~\ref{fig:addtional}. Manually setting these task-shared
parameters did not affect new class learning (matching baseline
performance). However, this configuration, when combined with SPI,
resulted in catastrophic forgetting, evidenced by a 37\% drop in
mIoU. By intentionally disrupting the SPI settings, we observed that
while new classes learned normally, old classes experienced
catastrophic forgetting. This observation, in reverse, further
substantiates the correctness of our proposed method. What's more,
adding class-specific LoRA to the backbone, the improvement obtained
by our method is not significant (only improve 1\% mIoU), but the
parameters need to be changed with the task, which will significantly
increase the inference time.

\end{document}